\documentclass[a4paper, 11pt]{article}

\usepackage[T1]{fontenc}
\usepackage{fullpage}
\usepackage{graphicx}
\usepackage{caption}
\usepackage{algorithm}
\usepackage{algorithmic}
\usepackage{color}
\usepackage{amsmath,amssymb,amsthm}
\usepackage{multirow}
\usepackage{microtype}
\usepackage{centernot}
\usepackage[hidelinks]{hyperref}
\usepackage[nameinlink]{cleveref}
\usepackage{booktabs}

\theoremstyle{plain}
\newtheorem{theorem}{Theorem}[]
\newtheorem{proposition}[theorem]{Proposition}
\newtheorem{lemma}[theorem]{Lemma}
\newtheorem{corollary}[theorem]{Corollary}

\theoremstyle{definition}
\newtheorem{definition}[theorem]{Definition}
\newtheorem{assumption}[theorem]{Assumption}
\theoremstyle{remark}
\newtheorem{remark}[theorem]{Remark}

\DeclareMathOperator*\Exp{\bf E}
\DeclareMathOperator*\Prob{\bf Pr}

\DeclareMathOperator*\Tr{\rm Tr}

\newcommand{\prn}[1]{\left(#1\right)}
\newcommand{\cprn}[1]{\!\left(#1\right)}
\newcommand{\sqbra}[1]{\left[#1\right]}
\newcommand{\csqbra}[1]{\!\left[#1\right]}

\newcommand{\abs}[1]{\left|#1\right|}

\newcommand{\brkts}[1]{\left\{#1\right\}}

\newcommand{\angles}[1]{\left\langle#1\right\rangle}
\newcommand{\bars}[1]{\left\|#1\right\|}

\newcommand{\cpr}[1]{\Prob\csqbra{#1}}

\newcommand{\cexpt}[1]{\Exp\csqbra{#1}}

\newcommand{\cexptq}[2]{\Exp_{#2}\csqbra{#1}}

\newcommand{\N}{\mathbb{N}}

\renewcommand{\P}{\mathbb{P}}

\newcommand{\R}{\mathbb{R}}

\newcommand{\cA}{\mathcal{A}}

\newcommand{\cD}{\mathcal{D}}

\newcommand{\cI}{\mathcal{I}}

\newcommand{\cN}{\mathcal{N}}

\newcommand{\cS}{\mathcal{S}}

\allowdisplaybreaks

\begin{document}

\title{\bf Efficient, Low-Regret, Online Reinforcement Learning for Linear MDPs}

\author{
Philips George John \\
National University of Singapore
\and
Arnab Bhattacharyya \\
National University of Singapore
\and
Silviu Maniu \\
Universit\'e Grenoble Alpes
\and
Dimitrios Myrisiotis \\
CNRS@CREATE LTD.
\and
Zhenan Wu \\
Huawei Shanghai Research Center
}

\maketitle

\begin{abstract}
Reinforcement learning algorithms are usually stated without theoretical guarantees regarding their performance.
Recently, Jin, Yang, Wang, and Jordan (COLT 2020) showed a polynomial-time reinforcement learning algorithm (namely, LSVI-UCB) for the setting of linear Markov decision processes, and provided theoretical guarantees regarding its running time and regret.
In real-world scenarios, however, the space usage of this algorithm can be prohibitive due to a utilized linear regression step.
We propose and analyze two modifications of LSVI-UCB, which alternate periods of learning and not-learning, to reduce space and time usage while maintaining sublinear regret.
We show experimentally, on synthetic data and real-world benchmarks, that our algorithms achieve low space usage and running time, while not significantly sacrificing regret.
\end{abstract}

\section{Introduction}

\label{sec:introduction}

In reinforcement learning (RL), an agent interacts with an environment by taking actions and receiving rewards.
The objective is to either reach a goal, or maximize cumulative reward over time.
Some well known applications are robot control, game playing, and recommendation systems.
The main components of reinforcement learning are the agent and the environment, which continuously interact with each other.
These algorithms learn by trial and error, and the agent's behavior is shaped by the rewards it receives.

Research on RL has seen many breakthroughs since the Deep Q-Learning paper~\cite{mnih2013playing}.
\cite{mnih2015human} introduced the Deep Q-Network (DQN) algorithm, which utilized a convolutional neural network to approximate the action-value function $Q$.
\cite{mnih2016asynchronous} introduced the Asynchronous Advantage Actor-Critic (A3C) algorithm, which enabled parallelization of RL training across multiple CPU cores.
\cite{schulman2017proximal} introduced the Proximal Policy Optimization (PPO) algorithm, which is a policy optimization method for RL.
PPO strikes a balance between sample efficiency and stability, making it widely adopted in various RL applications.
\cite{silver2018mastering} introduced AlphaZero, an RL algorithm that achieved superhuman performance in the game of Go, chess, and shogi.
AlphaZero demonstrated the power of RL in learning complex strategies without prior human knowledge.

Most RL algorithms use a lot of space, prohibiting their use in cases where the available memory is limited.
Some approaches, such as aggregated memory for RL, aim to improve sample efficiency and resilience to noise by using order-invariant functions in conjunction with standard memory modules.
Additionally, the size of the replay buffer is a crucial hyperparameter in experience replay, which can significantly impact the speed of learning and the quality of the resulting policy.

Closer to our focus, \cite{jin2023provably} introduced and analyzed the Least-Squares Value Iteration with Upper Confidence Bound (LSVI-UCB) algorithm.
The original algorithm, namely LSVI, uses \emph{linear function approximation} for the design of optimal policies for problems with a \emph{very large state space}.
LSVI-UCB adds \emph{provable} performance guarantees based on the notion of \emph{regret}, generally used in the context on multi-armed bandits.
Currently, this is among the few approaches that provide theoretical bounds on RL regret.
The issue with LSVI-type algorithms, and more generally those that are based on linear models, is that they require a large amount of memory to store the samples and perform regressions for estimating the $Q$-values.
This is a significant limitation in practice, as the memory requirements can be prohibitive for large-scale problems, or when one desires to implement these algorithms on low-powered embedded devices, such as those used in Internet of Things scenarios.

\subsection{Our Contributions}

In this paper we propose two variants of the LSVI-UCB algorithm whose objective is to minimize space usage.
The reasons why we focus on LSVI-UCB, compared to other existing approaches, are as follows:
(i) The model used by the LSVI-UCB, linear regression, is simple but still widely applicable in real-world scenarios; and
(ii) LSVI-UCB is an approach that is founded on theoretical guarantees~\cite{jin2023provably}, something that many of the heuristic approaches currently employed in practice (such as those based on neural networks, for example) are lacking.

To achieve better space and time efficiency, we modify LSVI-UCB to alternate periods of learning with periods where samples are discarded, along with a careful reset of the workspace according to the number of episodes that have elapsed.

In particular, our main contributions are as follows:
\begin{enumerate}
\item
We prove that an LSVI-UCB variant where the learning intervals are of \emph{fixed size}, \textsc{LSVI-UCB-Fixed} (\Cref{alg:LSVI-UCB_fixed_reset}), achieves a practical space-to-regret trade-off by appropriately resetting the work space after a fixed number of episodes has passed.
This variant yields better running time and space usage than LSVI-UCB, while maintaining sublinear regret.
\item
The second variant that we introduce \emph{adapts} the length of the learning intervals by keeping track of a proxy measure, that is faster to compute compared to the actual regression step, based on the deviation of the projection matrix used in LSVI-UCB.
We show that this variant, namely \textsc{LSVI-UCB-Adaptive} (see \Cref{alg:LSVI-UCB_adaptive_reset}), minimizes space usage as well and that it has better running time than LSVI-UCB.
As in \Cref{alg:LSVI-UCB_fixed_reset}, this algorithm too utilizes workspace resets that alleviate space usage and empirically keeps regret within reasonable bounds.
\item
We show experimentally that our algorithms minimize space usage and keep regret low, even for arbitrary linear MDPs.
\end{enumerate}
Regarding the significance of the problems that we are considering, one might wonder whether linear MDPs exist in real-world scenarios.
Orthogonally to the above contributions, we take first steps in this direction by taking a real-world benchmark and linearizing (that is, transforming the space-action state to a linear one, by learning a linear relation) it to show that our algorithm can be applied to more general benchmarks (see \Cref{sec:linear}).
As an aside, linearizing real-world problems may be a problem of independent interest.

\subsection{Paper Outline}

We present some background material in \Cref{sec:preliminaries}.
\Cref{sec:overview} presents a technical overview of our proposed algorithms.
We then present experimental evidence about the effectiveness of our proposed algorithms in \Cref{sec:experiments}.
We conclude in \Cref{sec:conclusion} with some interesting open problems.
Regarding our appendices, \Cref{sec:motivation} gives some further motivation regarding \textsc{LSVI-UCB-Adaptive}, and \Cref{sec:plots} includes some extra figures.

\section{Other Related Work}

\label{sec:related-work}

We continue in this section the discussion of related work that we started in \Cref{sec:introduction}.

\paragraph{Linear MDPs.}

\cite{li2021sample-efficient} focus on a scenario with value-based linear representation, which postulates linear realizability of the optimal Q-function (also called the ``linear $Q^*$ problem'').
They make progress towards understanding this linear $Q^*$ problem by investigating a new sampling protocol, which draws samples in an online/exploratory fashion but allows one to backtrack and revisit previous states.
\cite{papini2021reinforcement} study the role of the representation of state-action value functions in regret minimization in finite-horizon MDPs with linear structure.
\cite{wagenmaker2022first} obtain first-order regret bounds scaling not as the worst-case but with some measure of the performance of the optimal policy in the linear MDP setting.
\cite{huang2022towards} propose a formulation for deployment-efficient reinforcement learning (DE-RL) from an ``optimization with constraints'' perspective.
They are interested in exploring an MDP and obtaining a near-optimal policy within minimal deployment complexity, whereas in each deployment the policy can sample a large batch of data.

\cite{zhang2022making} study an alternative definition of linear MDPs that automatically ensures normalization while allowing efficient representation of learning via contrastive estimation.
\cite{he2023nearly} focus on episodic time-inhomogeneous linear Markov decision processes whose transition probability can be parameterized as a linear function of a given feature mapping, and they propose the first computationally efficient algorithm that achieves nearly optimal regret.

\paragraph{UCRL and UCRL2.}

\cite{auer2006logarithmic} present a learning algorithm (UCRL) for undiscounted reinforcement learning.
They use upper confidence bounds to show that UCRL achieves logarithmic online regret in the number of steps taken with respect to an optimal policy.
\cite{auer2008nearOptimal} present a variant of the UCRL algorithm, namely UCRL2.
UCRL2 defines a set of statistically plausible MDPs given the observations so far, and chooses an optimistic MDP (with respect to the achievable average reward) among these plausible MDPs.
Then it executes a policy which is (nearly) optimal for the optimistic MDP.

\paragraph{Bandits.}

A simpler setting than that of reinforcement learning is multi-armed-bandits, where multiple states are not considered (that is, there is only one state).
However, many of the techniques developed for bandits can be adapted to reinforcement learning, especially the notion of regret for linear bandits.
A comprehensive resource on bandits is~\cite{lattimore2020bandit}.

\cite{preil2023genetic} propose a new algorithm that combines concepts from the reinforcement learning domain of multi-armed bandits and random search strategies from the domain of genetic algorithms to solve discrete stochastic optimization problems via simulation.
In particular, their focus is on noisy large-scale problems, which often involve a multitude of dimensions as well as multiple local optima.

\cite{duckworth2023reinforcement} consider the challenging scenario of contextual bandits with continuous actions and large input ``context'' spaces, e.g., images.
They posit that by modifying reinforcement learning algorithms for continuous control, they can outperform hand-crafted contextual bandit algorithms for continuous actions on standard benchmark datasets, i.e. vector contexts.
\cite{baihan2024reinforcement} presents an overview of recent advancements of reinforcement learning and bandits, and discusses how they can be effectively employed to solve speech and natural language processing problems with models that are adaptive and scalable.

\paragraph{Sketching.}

One possible approach to reducing the space complexity of reinforcement learning is to use sketching, especially in the case of linear MDPs.
\cite{bellemare2012sketch-based} investigate the application of hashing to linear value function approximation.
\cite{andreas2017modular} describe a framework for multitask deep reinforcement learning guided by policy sketches.
To learn from sketches, they present a model that associates every subtask with a modular subpolicy, and jointly maximizes reward over full task-specific policies by tying parameters across shared subpolicies.
\cite{tai2018sketching} introduce a new sub-linear space sketch --- the Weight-Median Sketch --- for learning compressed linear classifiers over data streams while supporting the efficient recovery of large-magnitude weights in the model.
This enables memory-limited execution of several statistical analyses over streams, including online feature selection, streaming data explanation, relative deltoid detection, and streaming estimation of pointwise mutual information.
\cite{song2023sketching} propose a novel sketching scheme for the first order method in large-scale distributed learning setting, such that the communication costs between distributed agents are saved while the convergence of the algorithms is still guaranteed.
Using their framework, they develop algorithms for federated learning with lower communication costs.

\paragraph{Other LSVI or LSVI-UCB Extensions.}

\cite{shrivastava2021sublinear} present the first provable Least-Squares Value Iteration (LSVI) algorithm that achieves runtime complexity sub-linear in the number of actions.
They formulate the value function estimation procedure in value iteration as an approximate maximum inner product search problem and propose a locality sensitive hashing (LSH) type data structure to solve this problem with sublinear time complexity.
\cite{zhou2022nonstationary} consider reinforcement learning in episodic MDPs with linear function approximation under drifting environments.
They first develop the LSVI-UCB-Restart algorithm, an optimistic modification of least-squares value iteration combined with periodic restart, and establish its dynamic regret bound when variation budgets are known.
They then propose a parameter-free algorithm, Ada-LSVI-UCB-Restart, that works without knowing the variation budgets, but with a slightly worse dynamic regret bound.

\paragraph{Adaptivity.}

\cite{wang2021provablyEfficient} study RL with linear function approximation under the adaptivity constraint.
They consider two popular limited adaptivity models:
The batch learning model and the rare policy switch model, and propose two efficient online RL algorithms for episodic linear Markov decision processes, where the transition probability and the reward function can be represented as a linear function of some known feature mapping.
\cite{xiong2024aGeneralFramework} take the first step in studying general sequential decision-making under the two adaptivity constraints.
They provide a general class called the Eluder Condition (EC) class, which includes a wide range of reinforcement learning classes.
Then, for the rare policy switch constraint, they provide a generic algorithm to achieve a $\widetilde{O}\cprn{\log K}$ regret on the EC class. For the batch learning constraint, they provide an algorithm that provides a $\widetilde{O}\cprn{\sqrt{K} + K / B}$ regret with the number of batches $B$.

\section{Preliminaries}

\label{sec:preliminaries}

Let $A$ be a matrix.
We denote by $A^\top$ the \emph{transpose of $A$}, by $A^{-1}$ the \emph{inverse of $A$}, and by $\bars{A}_{\rm F}$ the \emph{Frobenius norm of $A$}.
For a vector $v$, we denote by $v^\top$ the \emph{transpose of $v$}.
For vectors $x,y$, we denote the \emph{dot product of $x$ and $y$} by $\angles{x,y}$.

\begin{theorem}[Sherman-Morrison]
\label{thm:sherman-morrison}
Suppose $A \in \R^{n \times n}$ is an invertible square matrix and $u, v \in \mathbb{R}^n$ are vectors.
Then $A + u v^\top$ is invertible if and only if $1 + v^\top A^{-1} u \neq 0$ and
\[
\left(A+uv^\top\right)^{-1}
= A^{-1} - \prn{A^{-1}uv^\top A^{-1}}/\prn{1+v^\top A^{-1}u}.
\]
\end{theorem}

We denote by $\Delta([d])$ the $(d-1)$-dimensional simplex in $\R^d$, which is the set of all vectors in $\R^d$ with non-negative entries that sum to $1$.
We use standard asymptotic notation, such as $O\cprn{\cdot}$ and $\Omega\cprn{\cdot}$, whereby $\widetilde O\cprn{\cdot}$ and $\widetilde\Omega\cprn{\cdot}$ hide logarithmic factors and $O_d\cprn{\cdot}$ and $\Omega_d\cprn{\cdot}$ imply that the underlying constant depends on some parameter $d$.

\subsection{Markov Decision Processes (MDPs)}

We adapt here the formalization of~\cite{jin2023provably}.
We define an \emph{episodic Markov Decision Process (episodic MDP)} as a tuple $(\cS, \cA, H, \P, r)$, where
$\cS$ is the set of \emph{states},
$\cA$ is the set of possible \emph{actions},
$H \in \N$ is the \emph{planning horizon},
$\P = \{\P_h\}_{h=0}^{H-1}$, $\P_h: S \times A \rightarrow \Delta(\cS)$, is the set of \emph{transition functions}; where $P_h$ maps each state-action pair $(s_h, a_h)$ to a distribution over $\cS$ which is used to sample the next state $s_{h+1}$,
$r = \{r_h\}_{h=0}^{H-1}$, $r_h: \cS \times \cA \rightarrow \Delta([0, 1])$, is the set of \emph{reward functions}, where $R_h$ maps a state-action pair $(s_h, a_h)$ to a distribution over $[0, 1]$ which is used to sample the stochastic reward $r_h$.

For a \textit{linear MDP} that we consider in this paper, we use the definition of~\cite{jin2023provably}, i.e., an MDP whose transition and reward functions for each step $h$ can be expressed as the dot product of a mapping depending only on state and action with a step-specific vector.

\begin{definition}[Linear MDP]
\label{def::prelims::linear_mdp}
The tuple $(\cS, \cA, H, \P, r)$ is a linear MDP with a feature map $\phi : \cS \times \cA \to \R^d$, if for any $h \in [H]$, there exist $d$ unknown (signed) measures $\mu_h = (\mu_{h,1}, \dots , \mu_{h,d})$ over $\cS$ and an unknown vector $\theta_h \in \R^d$, such that for any $(x, a) \in \cS \times \cA$, we have
$\P_h\cprn{\cdot|x,a}
=\angles{\phi\cprn{x,a},\mu_h}$ and $r_h\cprn{x,a} = \angles{\phi\cprn{x,a},\theta_h}$.
\end{definition}

Note that since the reward functions are bounded in $[0, 1]$, the value functions are bounded in $[0, H]$.
Moreover, we set $T := H K$.

We are interested in analyzing how an agent interacts with some MDP.
Initially, the agent starts at some initial state $s_1$ chosen by an adversary.
At each step $h \in [H]$, the agent observes the state $s_h \in \cS$, picks an action $a_h \in \cA$, receives a reward $r_h(s_h, a_h)$, and the MDP evolves into a new state $s_{h+1}$ that is drawn from the probability measure $\P_h(\cdot|s_h, a_h)$.
The episode terminates when $s_{H+1}$ is reached.

A \emph{policy} $\pi: \cS \times H \to \cA$ is a mapping such that $\pi(s, h)$ gives a probability distribution on $\cA$ conditioned on the MDP state being at state $s$ at the $h$-th step.
For each $h \in [H]$, we define the value function $V_h^\pi : \cS \to \R$ as the expected value of cumulative rewards received under policy $\pi$ when starting from some arbitrary state at the $h$-th step.
That is,
\[
V_h^\pi\cprn{x}
:=\cexpt{\sum_{h'=h}^H r_{h'}\cprn{s_{h'},\pi\cprn{s_{h'},h'}}|s_h=x}
\]
for all $x\in\cS$ and $h\in\sqbra{h}$.
We also define the action-value function $Q^\pi : \cS \times \cA \to \R$ as the expected cumulative rewards when policy $\pi$ comes in only one step later, and we are fixed with action $a$ at current state $s$.
That is,
\begin{align*}
Q_h^\pi\cprn{s,a}
:= r_h\cprn{s,a}
+ \cexpt{\sum_{h'=h+1}^H r_{h'}\cprn{s_{h'},\pi\cprn{s_{h'},h'}}|s_h=s,a_h=a}
\end{align*}
for all $\prn{s,a}\in\cS\times\cA$ and $h\in\sqbra{H}$.

An important property of linear MDPs is that, for any policy, the action-value functions are linear in the feature map $\phi$, so it suffices to focus on linear action-value functions.

\begin{proposition}[\cite{jin2023provably}]
\label{prop:linear_action-value}
For a linear MDP, for any policy $\pi$, there exist weights $\brkts{w_h^\pi}_h\in[H]$ such that for any $(s,a,h) \in \cS \times \cA \times [H]$, we have $Q_h^\pi(s,a) = \angles{\phi\cprn{s,a},w_h^\pi}$.
\end{proposition}

Since the action spaces and the episode length are both finite, there always exists an optimal policy $\pi$ which gives the optimal value $V_h^*(s) = \sup_\pi V_h^\pi(s)$ for all $s \in \cS$ and $h \in [H]$.
The existence of such optimal policies is a topological consequence of fixed point theorems applied to the Bellman transformations~\cite{puterman2014markov}.

\emph{Bellman equations} capture the conditions under which a policy is optimal.
They are as follows:
\begin{align*}
Q_h^*\cprn{s,a}
&=r_h\cprn{s,a}+\cexptq{V_{h+1}^*\cprn{s'}}{s'\sim\P_h\prn{\cdot|s,a}};
\quad
V_h^*\cprn{s}
&=\max_{a\in\cA}Q_h^*\cprn{s,a};
\quad
V_{H+1}^*\cprn{s}
=0.
\end{align*}
This implies that the optimal policy $\pi$ is a greedy policy with respect to the optimal action-value function $\brkts{Q_h}_{h\in[H]}$.

We shall now consider the total accumulated regret of an interaction between an agent and the environment.
For every $k \geq 1$ the adversary picks the initial state $s_{1,k}$ and the agent chooses policy $\pi_k$.
The difference between $V_1^{\pi_k}(s_{1,k})$ and $V_1(s_{1,k})$ is the expected regret of the agent at the $k$-th episode.
Therefore, the total expected regret is equal to
\(
\sum_{k=1}^K \prn{V_1^*\cprn{s_{1,k}} - V_1^{\pi_k}\cprn{s_{1,k}}}.
\)

\section{Technical Contributions}

\label{sec:overview}

We give a high-level overview of the main technical ingredients of LSVI-UCB and of our proposed algorithms.

\subsection{Background: LSVI-UCB}

As mentioned earlier, \cite{jin2023provably} showed that LSVI-UCB can handle problems with a \emph{very large state space}, that runs in \emph{polynomial time} and has \emph{provable} performance guarantees.
This is \Cref{alg:LSVI-UCB}.

\begin{algorithm}[htbt]
\caption{LSVI-UCB~\cite{jin2023provably}.}
\label{alg:LSVI-UCB}
\begin{algorithmic}[1]
\STATE {\bfseries Input:} Access to an MDP, parameters $K$, $H$, $\beta$, $\lambda$.
\STATE {\bfseries Output:} A sequence of policies.
\FOR{episode $k := 1,\dots,K$}
\FOR{step $h := H,\dots,1$}
\label{LSVI-UCB_step-4}
\STATE $\Lambda_{h} := \sum_{i = 1}^{k - 1} \phi\cprn{s_{h,i},a_{h,i}}\phi\cprn{s_{h,i},a_{h,i}}^\top+\lambda I$
\label{LSVI-UCB_step-5}
\STATE $u_{h} := \sum_{i = 1}^{k - 1}\phi\cprn{s_{h,i},a_{h,i}} r_h\cprn{s_{h,i},a_{h,i}}$
\label{LSVI-UCB_step-6}
\STATE $q_i := \max_{a\in\cA} Q_{h+1}\cprn{s_{h+1,i},a}$ for $i \in [k-1]$
\STATE $z_{h}:= \sum_{i = 1}^{k - 1}\phi\cprn{s_{h,i},a_{h,i}} q_i$
\label{LSVI-UCB_step-7}
\STATE $w_{h} := \Lambda_{h}^{-1} \prn{u_{h} + z_{h}}$
\STATE $M := w_{h}^\top \phi\cprn{\cdot,\cdot} + \beta\prn{\phi\cprn{\cdot,\cdot}\Lambda_{h}^{-1}\phi\cprn{\cdot,\cdot}}^{1/2}$
\STATE $Q_{h}\cprn{\cdot,\cdot} := \min\cprn{M,H}$
\label{LSVI-UCB_step-10}
\ENDFOR
\label{LSVI-UCB_step-11}
\STATE Receive the initial state $s_{1,k}$.
\FOR{step $h := 1,\dots,H$}
\label{LSVI-UCB_step-13}
\STATE Take action $a_{h,k} := \arg\max_{a\in\cA} Q_{h}\cprn{s_{h,k},a}$, and observe $s_{h+1,k}$.
\ENDFOR
\label{LSVI-UCB_step-15}
\ENDFOR
\end{algorithmic}
\end{algorithm}

As the reader may observe, each episode $k$ consists of two loops over all steps.
(In what follows, throughout the paper, we use subscript $k$ to refer to the value of the respective quantity at episode $k$.)
The first pass (lines \ref{LSVI-UCB_step-4} to \ref{LSVI-UCB_step-11}) updates the parameters $(w_{h,k}, \Lambda_{h,k})$ that are used to form the action-value function $Q_{h,k}$.
The second pass (lines \ref{LSVI-UCB_step-13} to \ref{LSVI-UCB_step-15}) executes the greedy policy, namely $a_h = \arg\max_{a\in\cA} Q_{h,k}(s_h, a)$, according to the $Q_{h,k}$ obtained in the first pass.
Note that $Q_{H+1}(\cdot, \cdot) := 0$ since the agent receives no reward after the $H$-th step.
For the first episode (where $k = 1$), since the summations in lines \ref{LSVI-UCB_step-5}, \ref{LSVI-UCB_step-6} , and \ref{LSVI-UCB_step-7} are from $i = 1$ to $i = 0$, we have $\Lambda_{h,k} = \lambda I$ and $w_{h,k} = 0$.
Line \ref{LSVI-UCB_step-10} specifies the dependency of the action-value function $Q_{h,k}$ on the parameters $w_{h,k}$ and $\Lambda_{h, k}$.

\begin{remark}
\label{rem:LSVI-UCB}
\cite{jin2023provably} prove the following:
\Cref{alg:LSVI-UCB} runs in time $O_d\cprn{\abs{\cA}K^2}$ and uses space $O_d\cprn{\abs{\cA} K}$.
Moreover, its regret is $\widetilde{O}_d\cprn{\sqrt{K}}$ (with constant probability).
\end{remark}

The time bound comes from the following considerations.
If we compute $\Lambda_{h,k}^{-1}$ by the Sherman-Morrison formula (see \Cref{thm:sherman-morrison}), the running time of \Cref{alg:LSVI-UCB} is dominated by computing $\max_{a\in\cA} Q_{h+1,k}\cprn{s_{h+1,i},a}$ for all $i \in [k - 1]$.
This takes $O(d^2 \abs{\cA} K)$ time per step, which gives a total of $O(d^2 \abs{\cA} K T) = O_d\cprn{\abs{\cA}K^2}$, since there are $O\cprn{T} = O_d(K)$ steps.

The space bound comes from the observation that \Cref{alg:LSVI-UCB} needs to store the values of $\Lambda_{h,k}, w_{h,k}, r_h(s_{h,k}, a_{h,k})$ and $\brkts{\phi(s_{h,k}, a)}_{a\in\cA}$ for all $(h, k) \in [H] \times [K]$, which takes $O(d^2H + d\abs{\cA}T) = O_d\cprn{\abs{\cA} K}$ space.

\begin{remark}[On the space usage of LSVI-UCB]
As we mentioned, one can efficiently compute $\Lambda_{h,k}^{-1}$ from $\Lambda_{h,k-1}^{-1}$ (see \Cref{thm:sherman-morrison}), but the regression solution (episode $k$, step $h$) is $(\Lambda_{h,k})^{-1}\Phi_{h,k}^\top y_{h,k}$, whereby
\[
\Phi_{h,k}=\begin{bmatrix}\phi(s_{h,1},a_{h,1)}\\
\vdots\\
\phi(s_{h,k-1},a_{h,k-1})\end{bmatrix}\in\mathbb{R}^{(k-1)\times d},
\]
and $y_{h,k} \in \mathbb{R}^{(k-1)\times 1}$ is
\[
\begin{bmatrix}
r(s_{h,1},a_{h,1})+\max_{a\in\mathcal{A}}\widehat{Q}_{h+1,k}(s_{h,1},a_{h,1})\\\vdots\\r(s_{h,k-1},a_{h,k-1})+\max_{a\in\mathcal{A}}\widehat{Q}_{h+1,k}(s_{h,k-1},a_{h,k-1})
\end{bmatrix}.
\]
Using the solution above will incur linear in $K$ space, and that the previous $\phi$ vectors are necessary.
It is possible~\cite{yang2020reinforcement} to efficiently compute $(\Lambda_{h,k})^{-1}\Phi_{h,k}^\top$ using sketching, reducing the dimension from $d \times \prn{k-1}$ to $d \times{\rm poly}(d,\log k,1/\varepsilon)$.
However, to get $\varepsilon$-approximate regression using sketching~\cite{woodruff2014sketching}, one needs to sketch $y_{h,k}$ as well, and this is not possible efficiently (without recomputing the sketch on the $\prn{k-1}$-dimensional vector $y_{h,k}$) since $y_{h,k}$ does not depend linearly on $y_{h,k-1}$ due to $\widehat{Q}_{h+1,k}$ versus $\widehat{Q}_{h+1,k-1}$.
One can put the sketch ``inside'' $\widehat{Q}_{h+1,k}$, but that would not give any approximation guarantees for regression, and could in fact lead to an additional linear term in the regret, e.g., with the count-min sketch.
This helps understand why LSVI-UCB is costly in terms of space usage.
\end{remark}

\subsection{LSVI-UCB With Fixed Intervals}

\label{sec:overview-fixed_reset}

\Cref{alg:LSVI-UCB_fixed_reset} presents the pseudocode for LSVI-UCB with \emph{fixed} learning intervals.

\begin{algorithm}[htbt]
\caption{LSVI-UCB with \emph{fixed} learning intervals and \emph{reset}.}
\label{alg:LSVI-UCB_fixed_reset}
\begin{algorithmic}[1]
\STATE {\bfseries Input:} Access to an MDP, parameters $K$, $H$, $\beta$, $\lambda$, $\rho$.
\STATE {\bfseries Output:} A sequence of policies.
\STATE $I_h := \emptyset$ for $h \in [H]$
\STATE $K_0 := 0$
\FOR{episode $k := 1,\dots,K$}
\IF{$k < K_0 + K^{\rho}$}
\FOR{step $h := H,\dots,1$}
\STATE $\Lambda_{h} := \sum_{i \in I_h} \phi\cprn{s_{h,i}, a_{h,i}}\phi\cprn{s_{h,i}, a_{h,i}}^\top+\lambda I$
\STATE $w_{h}^{0} := \sum_{i \in I_h}\phi\cprn{s_{h,i}, a_{h,i}} r_h\cprn{s_{h,i}, a_{h,i}}$
\STATE $q_i := \max_{a\in\cA} Q_{h+1}\cprn{s_{h+1,i}, a}$ for $i \in I_h$
\label{line:dominant-cost-fixed_reset}
\STATE $w_{h}^{1}:= \sum_{i \in I_h} \phi\cprn{s_{h,i},a_{h,i}} q_i$
\STATE $w_{h} := \Lambda_{h}^{-1} \prn{w_{h}^{0} + w_{h}^{1}}$
\STATE $M := w_{h}^\top \phi\cprn{\cdot,\cdot} + \beta\prn{\phi\cprn{\cdot,\cdot}\Lambda_{h}^{-1}\phi\cprn{\cdot,\cdot}}^{1/2}$
\STATE $Q_{h}\cprn{\cdot,\cdot} := \min\cprn{M, H}$
\STATE $I_h := I_h \cup \{k\}$
\ENDFOR
\ELSIF{$k = K_0 + K^{\rho}$}
\STATE $K_0 := K_0 + K^{\rho}$
\STATE Delete the working space; $I_h := \emptyset$ for $h \in [H]$
\ENDIF
\STATE Receive the initial state $s_{1,k}$.
\FOR{step $h := 1,\dots,H$}
\STATE Take action $a_{h,k} := \arg\max_{a\in\cA} Q_{h,k}\cprn{s_{h,k},a}$, and observe $s_{h+1,k}$.
\ENDFOR
\ENDFOR
\end{algorithmic}
\end{algorithm}

This algorithm imitates LSVI-UCB and resets the workspace every $K^\rho$ episodes to save space.
In particular, we show in \Cref{prop:fixed_reset} that in this way we achieve a practical space-to-regret trade-off that yields sub-linear space usage and regret.

\begin{proposition}
\label{prop:fixed_reset}
\Cref{alg:LSVI-UCB_fixed_reset} runs in time $O_d\cprn{\abs{\cA} K^{1 + \rho}}$ and uses space $O_d\cprn{\abs{\cA} K^\rho}$.
Moreover, its regret is $\widetilde{O}_d\cprn{K^{1 - \rho / 2}}$ (with constant probability).
\end{proposition}

\begin{proof}[Proof (Sketch)]
The space usage of \Cref{alg:LSVI-UCB_fixed_reset} is $O_d\cprn{\abs{\cA} K^\rho}$ by adjusting the discussion regarding the space usage of \Cref{alg:LSVI-UCB}.
Under the same token, the regret bound of \Cref{alg:LSVI-UCB_fixed_reset} is $K^{1 - \rho} \cdot \widetilde{O}_d\cprn{\sqrt{K^\rho}} = \widetilde{O}_d\cprn{K^{1 - \rho / 2}}$ since there are $K^{1 - \rho}$ intervals for which \Cref{alg:LSVI-UCB_fixed_reset} incurs $\widetilde{O}_d\cprn{\sqrt{K^\rho}}$ regret (by the correctness of LSVI-UCB for $K^\rho$ episodes).
This yields the desired space-to-regret trade-off.
Finally, the running time of \Cref{alg:LSVI-UCB_fixed_reset} (by adjusting the respective discussion of \Cref{alg:LSVI-UCB}) is $O_d\cprn{\abs{\cA} K^{1 + \rho}}$, since each of the $O\cprn{T} = O_d\cprn{K}$ computations of lines 10 and 12 incurs a cost of $O\cprn{d^2 \abs{\cA} K^\rho} = O_d\cprn{\abs{\cA} K^\rho}$.
\end{proof}

Compared to \Cref{alg:LSVI-UCB}, \Cref{alg:LSVI-UCB_fixed_reset} uses less space (since $K^\rho \ll K$) and achieves smaller running time (since $K^{1 + \rho} \ll K^2$), while maintaining sublinear in $K$ regret.

\subsection{LSVI-UCB With Adaptive Intervals}

\label{sec:overview-adaptive_reset}

We also study an adaptive variant of LSVI-UCB (see \Cref{alg:LSVI-UCB_adaptive_reset}).

\begin{algorithm}[htbt]
\caption{LSVI-UCB with \emph{adaptive} learning intervals and \emph{reset}.}
\label{alg:LSVI-UCB_adaptive_reset}
\begin{algorithmic}[1]
\STATE {\bfseries Input:} Access to an MDP, parameters $K$, $H$, $\beta$, $\lambda$, $m$, $\tau$, $\mathsf{Budget}$, $\rho$.
\STATE {\bfseries Output:} A sequence of policies.
\STATE $I_h := \emptyset$; $\mathsf{LearnIts}[h], \mathsf{TotIts}[h] := 0$ for $h \in [H]$
\FOR{episode $k := 1,\dots,K$}
\FOR{step $h := H,\dots,1$}
\STATE $L_{h,k} := {\sum_{i = 1}^{k - 1} \phi\cprn{s_{h,i},a_{h,i}}\phi\cprn{s_{h,i},a_{h,i}}^\top+\lambda I}$
\STATE \COMMENT{We store $L_{h,k},L_{h,k-1},\ldots,L_{h,k-m}$ \emph{only}.}
\label{line:L_reset}
\IF{$\mathsf{LearnIts}[h] < \mathsf{Budget}$ {\bf and} $\mathsf{TotIts}[h] < K^\rho$}
\STATE $\mathsf{TotIts}[h] := \mathsf{TotIts}[h] + 1$
\IF{$\textsc{Learn}\cprn{\brkts{L_{h,i}^{-1}}_{i=k-m}^k,\tau}$}
\STATE $\mathsf{LearnIts}[h] := \mathsf{LearnIts}[h] + 1$
\STATE $\Lambda_{h} := {\sum_{i \in I_{h}} \phi(s_{h,i},a_{h,i}) \phi(s_{h,i}, a_{h,i})^\top + \lambda I}$
\label{line:Lambda_reset}
\STATE $w_{h}^{0} := \sum_{i \in I_{h}}\phi\cprn{s_{h,i},a_{h,i}} r_h\cprn{s_{h,i},a_{h,i}}$
\STATE $q_i := \max_{a\in\cA} Q_{h+1}\cprn{s_{h+1,i},a}$ for $i \in I_h$
\label{line:dominant-cost-adaptive_reset}
\STATE $w_{h}^{1}:= \sum_{i \in I_{h}}\phi\cprn{s_{h,i},a_{h,i}} q_i$
\STATE $w_{h} := \Lambda_{h}^{-1} \prn{w_{h}^{0} + w_{h}^{1}}$
\STATE $M := w_{h}^\top \phi\cprn{\cdot,\cdot} + \beta\prn{\phi\cprn{\cdot,\cdot}\Lambda_{h}^{-1}\phi\cprn{\cdot,\cdot}}^{1/2}$
\STATE $Q_{h}\cprn{\cdot,\cdot} := \min\cprn{M, H}$
\STATE $I_h := I_h \cup \{k\}$
\ENDIF
\ELSE
\STATE Delete the working space.
\STATE $I_h := \emptyset$; $\mathsf{LearnIts}[h], \mathsf{TotIts}[h] := 0$ for $h \in [H]$
\ENDIF
\ENDFOR
\STATE Receive the initial state $s_{1,k}$.
\FOR{step $h := 1,\dots,H$}
\STATE Take action $a_{h,k} := \arg\max_{a\in\cA} Q_{h}\cprn{s_{h,k},a}$, and observe $s_{h+1,k}$.
\ENDFOR
\ENDFOR
\end{algorithmic}
\end{algorithm}

The main difference here (compared to \Cref{alg:LSVI-UCB_fixed_reset}) is that it learns a new policy in step $h$ of episode $k$ only when (i) $\mathsf{LearnIts}[h]$ (which counts the number of episodes where learning happened in step $h$) is at most $\mathsf{Budget} \ll K$ and $\mathsf{TotIts}[h]$ (which counts the total number of episodes in step $h$) is at most $ K^\rho$, and (ii) the current projection matrix $\Lambda_{h,k}^{-1}$ (as used by LSVI-UCB) significantly deviates from the $\Lambda_{h,i}^{-1}$ used in the previous $m$ episodes, meaning that
\begin{equation}
\max_{i, j \in \{k - m, k - m + 1, \ldots, k\}}
\| \Lambda_{h,i}^{-1} - \Lambda_{h,j}^{-1} \|_{\rm F}
\geq \tau,
\label{eq:condition_reset}
\end{equation}
for some threshold $\tau$.
This condition is checked by
\(
\textsc{Learn}
\).
The motivation behind \Cref{eq:condition_reset} is drawn from the following observation (proved in \Cref{sec:motivation}).
 
\begin{proposition}
\label{prop:motivation}
The convergence of the sequence $\brkts{\Lambda_{h, k}^{-1}}_{h, k}$ in the operator norm implies that the action-value function learned by \Cref{alg:LSVI-UCB_adaptive_reset} is close in absolute value to the one learned by \Cref{alg:LSVI-UCB}.
Moreover, under some assumptions on the underlying MDP, the sequence $\brkts{\Lambda_{h, k}^{-1}}_{h, k}$ converges in the operator norm.
\end{proposition}

In a learning iteration of \Cref{alg:LSVI-UCB_adaptive_reset}, the learner learns a new policy exactly as in LSVI-UCB, but only using the data stored in the previous learning iterations.
In a non-learning iteration $(k, h)$, where $\textsc{Learn}$ returns False, the learner takes actions using the policy which was learnt in the last iteration where learning took place, but does not store any data other than $L_{h,k}$ (and this only for the next $m$ iterations).

\begin{proposition}
\label{prop:adaptive_reset}
\Cref{alg:LSVI-UCB_adaptive_reset} runs in time $O_d\cprn{\abs{\cA} K \mathsf{Budget}}$ and uses space $O_d(|\cA| \mathsf{Budget})$.
\end{proposition}

\begin{proof}[Proof (Sketch)]
The running time of \Cref{alg:LSVI-UCB_adaptive_reset} (by adjusting the discussion regarding the running time of \Cref{alg:LSVI-UCB}) is $O_d\cprn{\abs{\cA} K \mathsf{Budget}}$, since each of the $O\cprn{T} = O_d\cprn{K}$ computations of lines 14 and 16 incurs a cost of $O\cprn{d^2 \abs{\cA} \mathsf{Budget}} = O_d\cprn{\abs{\cA} \mathsf{Budget}}$.

Let us now argue about the space bound.
Suppose that $m = O(d)$.
Let $I_{\rm full}$ denote $\bigcup_{h \in [H]} I_{h}$ after $\mathsf{Budget}$ episodes.
By construction, we have $|I_{\rm full}| \leq O(H \mathsf{Budget})$.
Moreover, $L_{h,i}^{-1}$ can be computed efficiently from $L_{h,i-1}^{-1}$ using the Sherman-Morrison formula (as can $\Lambda_h^{-1}$), and the algorithm only needs to store at most $m+1$ of them, namely $\{L_{h,i}\}_{i=k-m}^{k}$, at any time.
The algorithm also needs to store the values of $\Lambda_{h}^{-1}$, $w_h$, as well as that of $\{\phi(s_{h,i},a)\}_{i \in I_{\rm full},h \in [H],a \in \cA}$ and $\{r_h(s_{h,i},a_{h,i})\}_{h \in [H], i \in I_h}$.
This yields a $O(d^2 H + |\cA|d H^2 \mathsf{Budget}) = O_d(|\cA| \mathsf{Budget})$ space bound, as claimed.
\end{proof}

That is, the space usage $O_d(|\cA| \mathsf{Budget})$ of \Cref{alg:LSVI-UCB_adaptive_reset} is much less compared to the respective $O_d(|\cA| K)$ bound of \Cref{alg:LSVI-UCB}, since $\mathsf{Budget} \ll K$.
The same holds true for the running time, as $O_d(|\cA| K \mathsf{Budget}) \ll O_d(|\cA| K^2)$.

Moreover, our experimental study in \Cref{sec:experiments} shows that \Cref{alg:LSVI-UCB_adaptive_reset} has a regret that is indistinguishable from that of \Cref{alg:LSVI-UCB} and \Cref{alg:LSVI-UCB_fixed_reset} (which is sublinear in $K$).

\section{Experiments}

\label{sec:experiments}

\subsection{Environments (Datasets)}

\label{sec:expts:envs}

We run our experiments on \emph{synthetic} linear MDPs (randomly generated parameters) as well as on \emph{linearized} versions of some standard RL environments (from OpenAI \textsf{gym}) with finite action spaces.
Finite action spaces are required for LSVI-UCB itself and also for our variants, which leads us to focus on the \textsf{ALE} environments (\textsf{Atari 2600} games).
We use the \textsf{-ram} version of these environments, where the state is the 128 byte RAM contents of the device rather than an RGB/grayscale image of the game screen.

We implement our RL algorithms in Python.
The algorithms themselves are implemented using only \textsf{numpy} and \textsf{scipy} with \textsf{joblib} used for CPU-level parallelization.
The linearization process and the representations used in the linearized environments are implemented using \textsf{pytorch} and benefit from GPU-level parallelization.

\subparagraph{Synthetic Linear MDP Environments}

We run most of our experiments on a synthetic linear MDP environment, with $\cS = [500]$, $\cA = [15]$, $d = 30$ (feature-embedding dimension) and $H = 50$ (planning horizon).
This synthetic MDP is generated with random features ($\phi(s,a)$ is sampled uniformly from the simplex $\Delta([d])$, independently for each $(s, a) \in \cS \times \cA$).
Similarly, the reward weight vectors ($\theta_h$, $h \in [H]$, such that $r_h(s,a) = \langle \theta_h, \phi(s,a)\rangle$) are also independently sampled uniformly at random from the $\Delta([d])$ simplex for each $h \in [H]$.
Finally, the transition-distribution measures $\mu_{h,1},\ldots,\mu_{h,d}$ (see \Cref{def::prelims::linear_mdp}) are also independently sampled uniformly from the $\Delta(\cS)$ simplex (i.e., they are discrete probability distributions on $\cS$) for each step $h$.

\subparagraph{Linearized Environments}

\label{sec:linear}

The linearization process is done using the Contrastive Representation Learning (CTRL) approach of~\cite{zhang2022making}.
We run their training algorithm with exploration for $50{,}000$ iterations (with $1{,}000$ iterations of pure exploration) and take the representations of $\phi$ (feature mapping), $\mu$ (transition measures mapping) and $\theta$ (reward weights) learnt at the end; the representations are in the form of neural networks with ReLu activations.
These representations give us a linear MDP environment that approximates the original environment.

We obtain the linear environments used in our experiments by \emph{linearizing}, as described above, the \textsf{ALE/Alien-ram-v5} and \textsf{ALE/Phoenix-ram-v5} environments~\cite{machado18arcade} from the \textsf{ale-py} package.
For brevity, we will refer to these \emph{linearized} environments as \textsf{Alien} and \textsf{Phoenix} respectively.
Adapting the implementation of~\cite{zhang2022making}, we use $d = 512$ while linearizing and take $H = 50$ while learning.

\begin{remark}
\cite{hwang2023modelBased} show that there is some state-action feature map that cannot induce linear MDPs, since when a feature map is provided, it is impossible to determine whether linear MDPs can be constructed using that map.
\end{remark}

\subsection{Algorithms}

\label{sec:expts:algs}

We consider the following online RL algorithms:
\begin{enumerate}
\item
\textbf{\textsc{LSVI-UCB}}:
This is the standard LSVI-UCB algorithm proposed in~\cite{jin2023provably}, which we use as a baseline for comparison.
This enables us to measure the regret vs. efficiency tradeoffs (both time and space) that we get when using our proposed algorithms.
LSVI-UCB has two hyperparameters, namely $\lambda$ and $\beta$, which appear in our algorithms as well.
We set them as proposed in~\cite{jin2023provably} to obtain sublinear $\widetilde{O}(\sqrt{K})$ regret bounds, without further studying their effect.
\item
\textbf{\textsc{LSVI-UCB-Fixed}}:
This is our proposed \Cref{alg:LSVI-UCB_fixed_reset}.
In the experiments, we try various values of $\rho \in [0.5,0.75]$ (phase length, i.e., the amount of episodes between resets, is $K^\rho$).
\item
\textbf{\textsc{LSVI-UCB-Adaptive}}:
This is our \Cref{alg:LSVI-UCB_adaptive_reset}.
In addition to $\lambda$ and $\beta$, it has these additional hyperparameters:
$m$ (the lookback period), $\tau$ (the pairwise-closeness threshold in the \textsc{Learn} condition), \textsf{Budget} (the maximum number of learning episodes in each phase, for each $h$), and $\rho$ (whereby $K^\rho$ is the maximum number of episodes in a phase).
We set $m$ to fixed values between $10$ and $50$, and we set $\tau := \tau_c \cdot d^2$ for various positive constants $\tau_c$ (this is because the matrices involved have $d^2$ entries and we consider closeness in the Frobenius norm).
We similarly set the \textsf{Budget} as $K^c$ for constants $c \in [0.5, 0.75]$, and set $\rho \geq c$ in this interval to measure the time-space-regret tradeoffs.
\end{enumerate}

\subsection{Synthetic Linear Environment Experiments}

First, we compare the regret of \textsc{LSVI-UCB-Fixed} and \textsc{LSVI-UCB-Adaptive} to the regret of LSVI-UCB as the baseline, in \Cref{fig:regret_synthetic}.

\begin{figure}
\centering
\includegraphics[width=0.65\columnwidth]{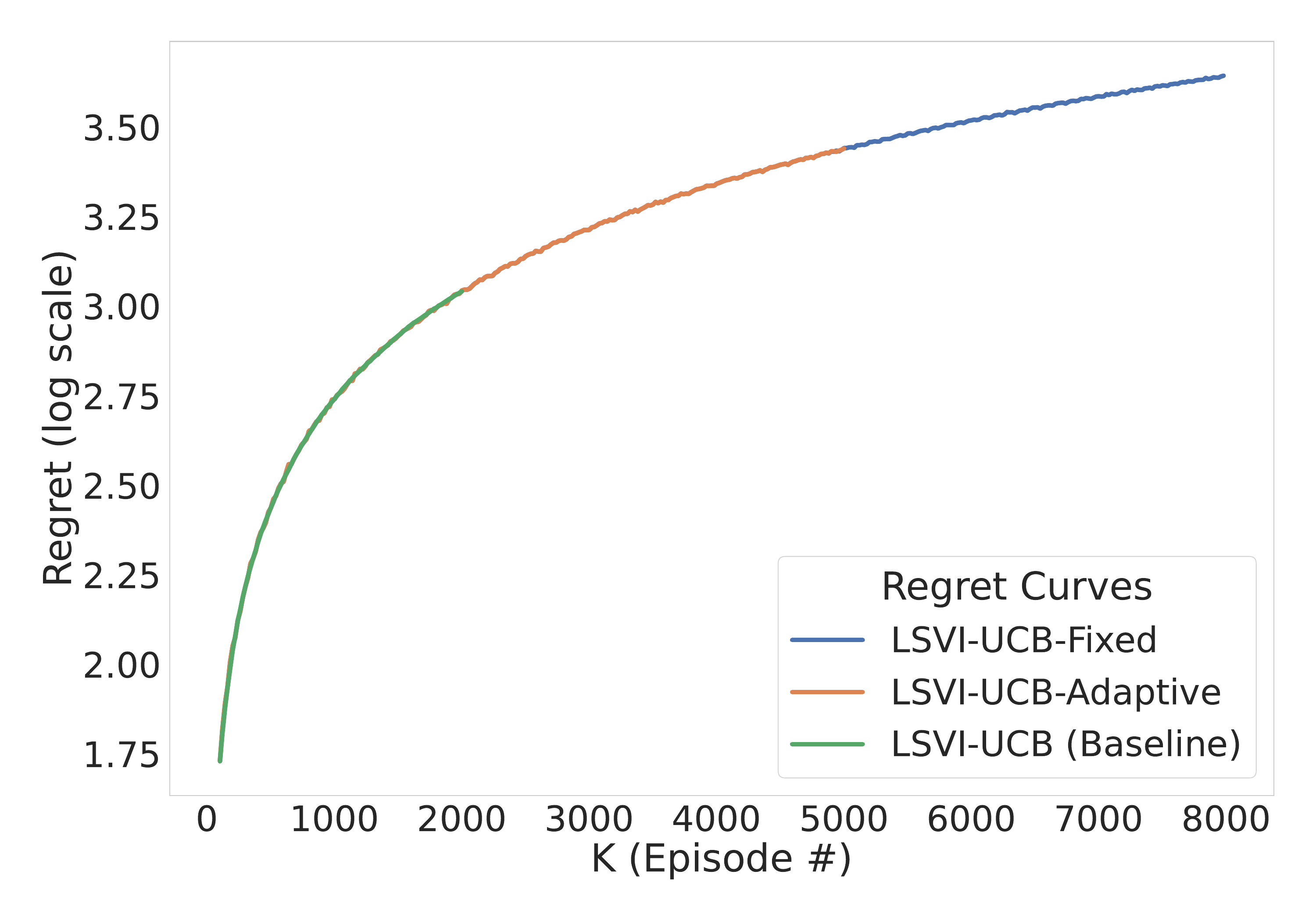}
\caption{Synthetic data: Regret curve.}
\label{fig:regret_synthetic}
\end{figure}

One can immediately notice that the cumulative regret is essentially indistinguishable from the regret achieved by LSVI-UCB, which is expected since our algorithms are designed to keep the regret low.

By design, both of our algorithms \textsc{LSVI-UCB-Fixed} and \textsc{LSVI-UCB-Adaptive}, are much more efficient in terms of space compared to LSVI-UCB.
Note that LSVI-UCB uses space which is linear in $K$ (the number of episodes), whereas our algorithms always use space which is sublinear in $K$ (exact space complexity depending on the parameters) --- see \Cref{sec:overview-fixed_reset} and \Cref{sec:overview-adaptive_reset}.
We get empirical verification of this on synthetic data, as well as an idea of the substantial improvement in the sizes of the data structures (\Cref{fig:space_synthetic}).

\begin{figure}
\centering
\includegraphics[width=0.65\columnwidth]{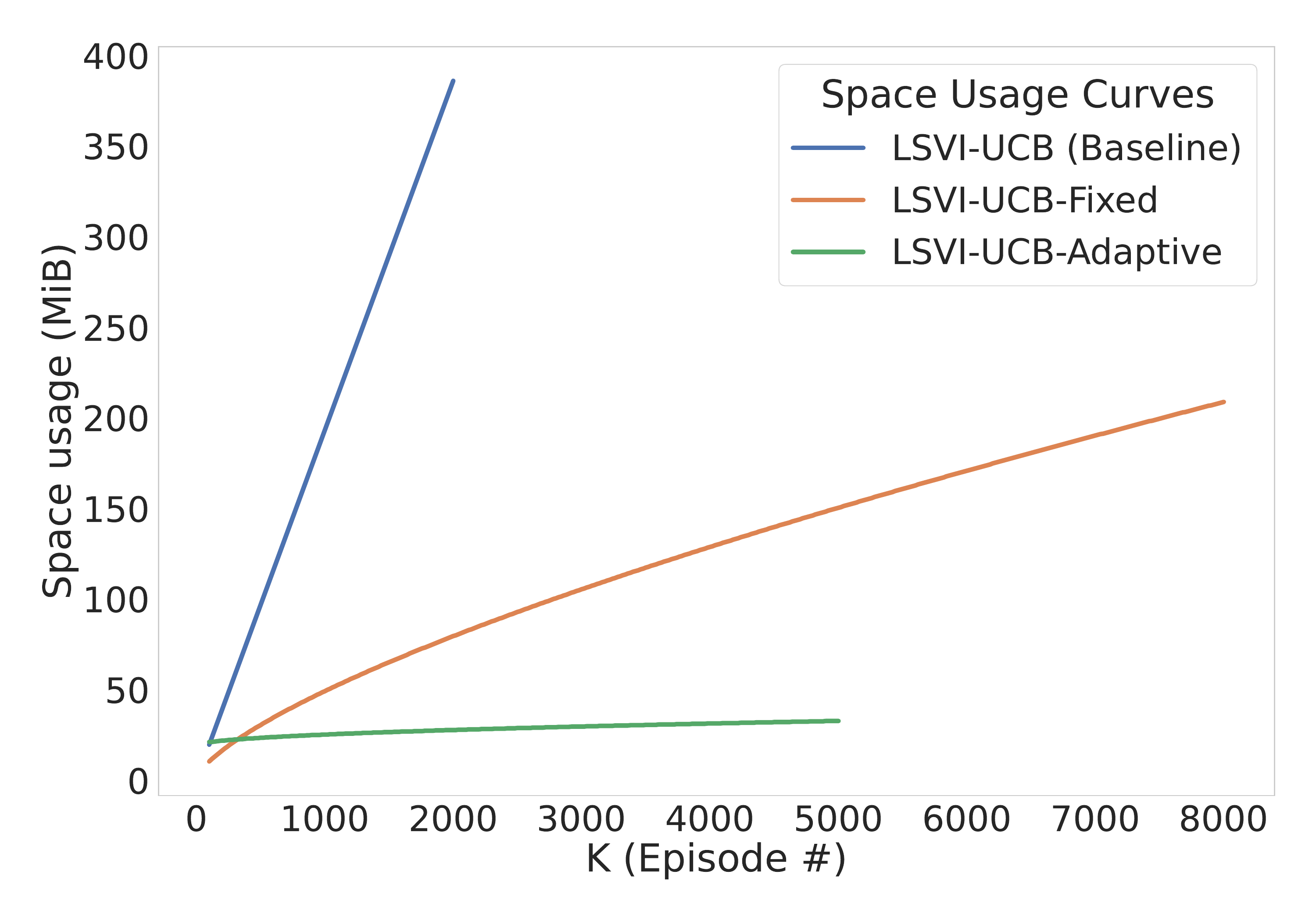}
\caption{Synthetic data: Space usage.}
\label{fig:space_synthetic}
\end{figure}

To further show the advantages of our algorithms in terms of efficiency, we compare the process execution (CPU) time of \textsc{LSVI-UCB-Adaptive} and of \textsc{LSVI-UCB-Fixed} to the process execution (CPU) time of LSVI-UCB (see \Cref{fig:time_synthetic}).

\begin{figure}
\centering
\includegraphics[width=0.65\columnwidth]{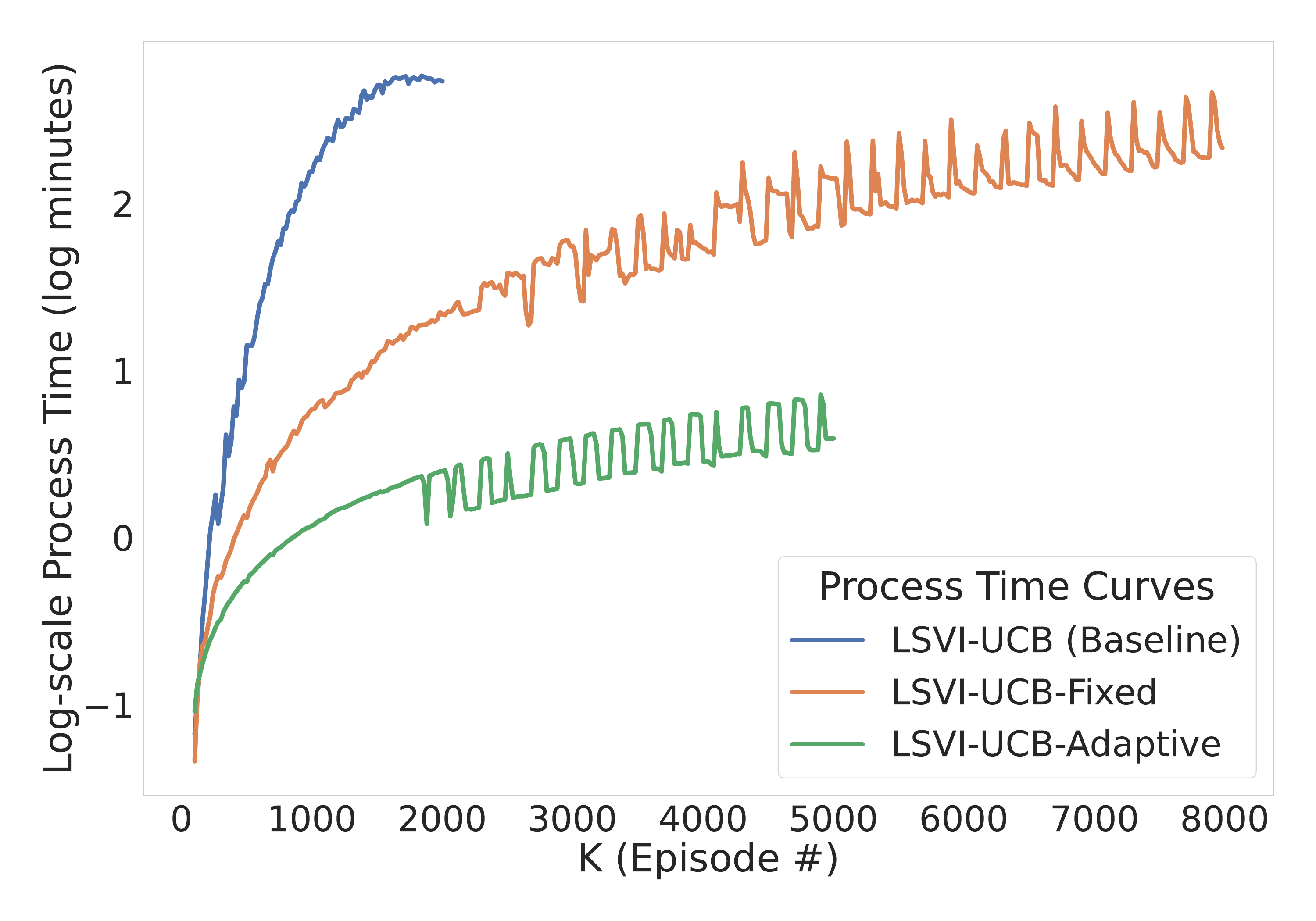}
\caption{Synthetic data: Running time.}
\label{fig:time_synthetic}
\end{figure}

We believe that the ``zig-zag'' phenomenon in the process-time curves of \Cref{fig:time_synthetic} (where we measure CPU time by Python's \texttt{time.process\_time}) is due to the fact that our experiments were ran on a heavily-loaded cluster with fewer GPUs, alongside other unrelated experiments, using multi-process parallelization again for various values of $K$, which could lead to starvation.

We observe that both of our proposed algorithms have running time that is asymptotically at most as that of LSVI-UCB.
This is because both of our algorithms end up using matrix operations (for the ridge regression step, etc.) with much smaller matrices when compared to LSVI-UCB.
This ends up more than offsetting the time costs of the additional bookkeeping involved in our algorithms (the $\textsc{Learn}$ conditions and the data structure to keep track of the proxies for non-learning intervals).

The space usage as it relates to the parameters is presented in \Cref{fig:space_fixed_param} and \Cref{fig:space_adapt_param}.

\begin{figure}
\centering
\includegraphics[width=0.65\columnwidth]{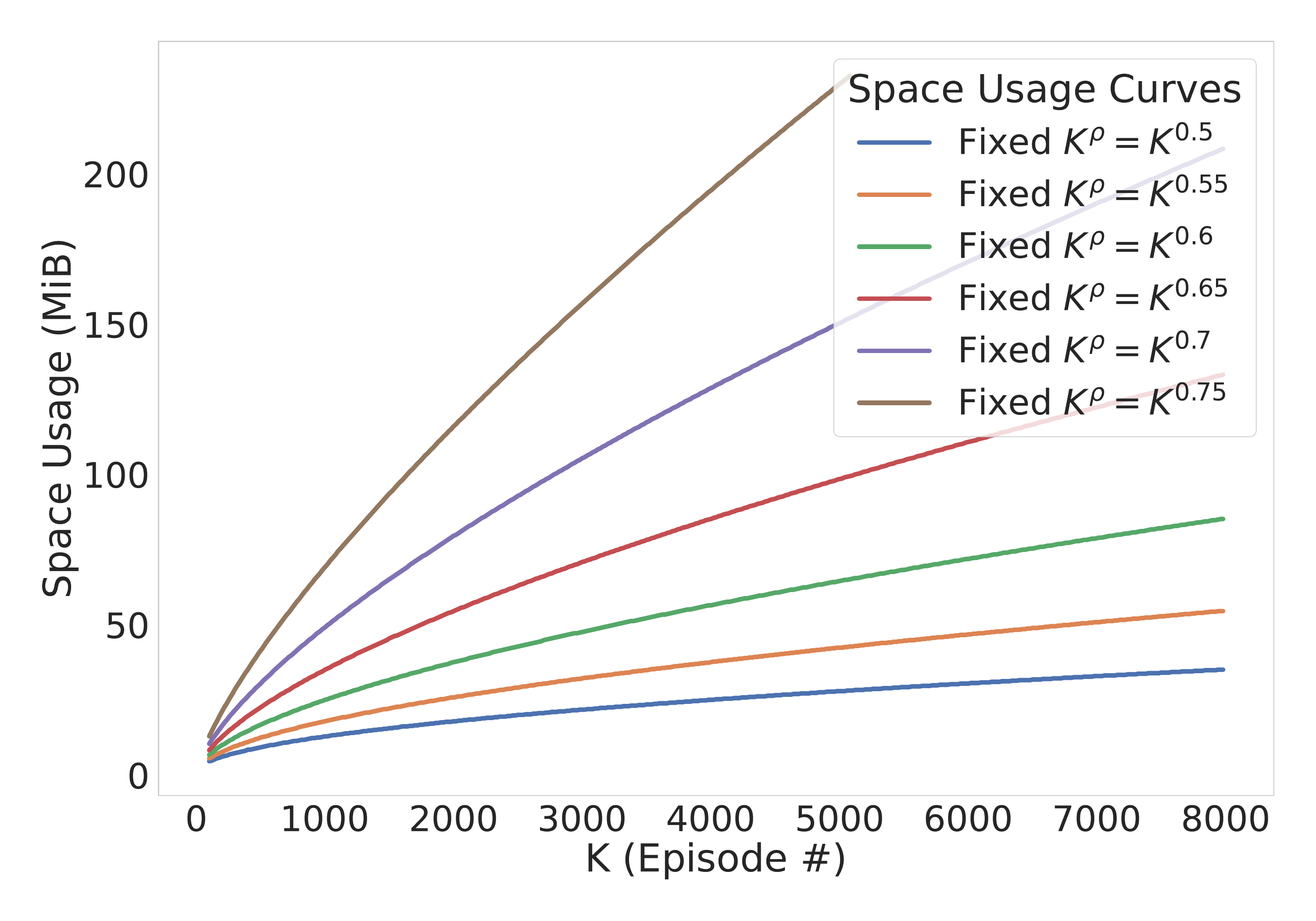}
\caption{Synthetic data: Space usage of \textsc{LSVI-UCB-Fixed} as a function of parameters.}
\label{fig:space_fixed_param}
\end{figure}

\begin{figure}
\centering
\includegraphics[width=0.65\columnwidth]{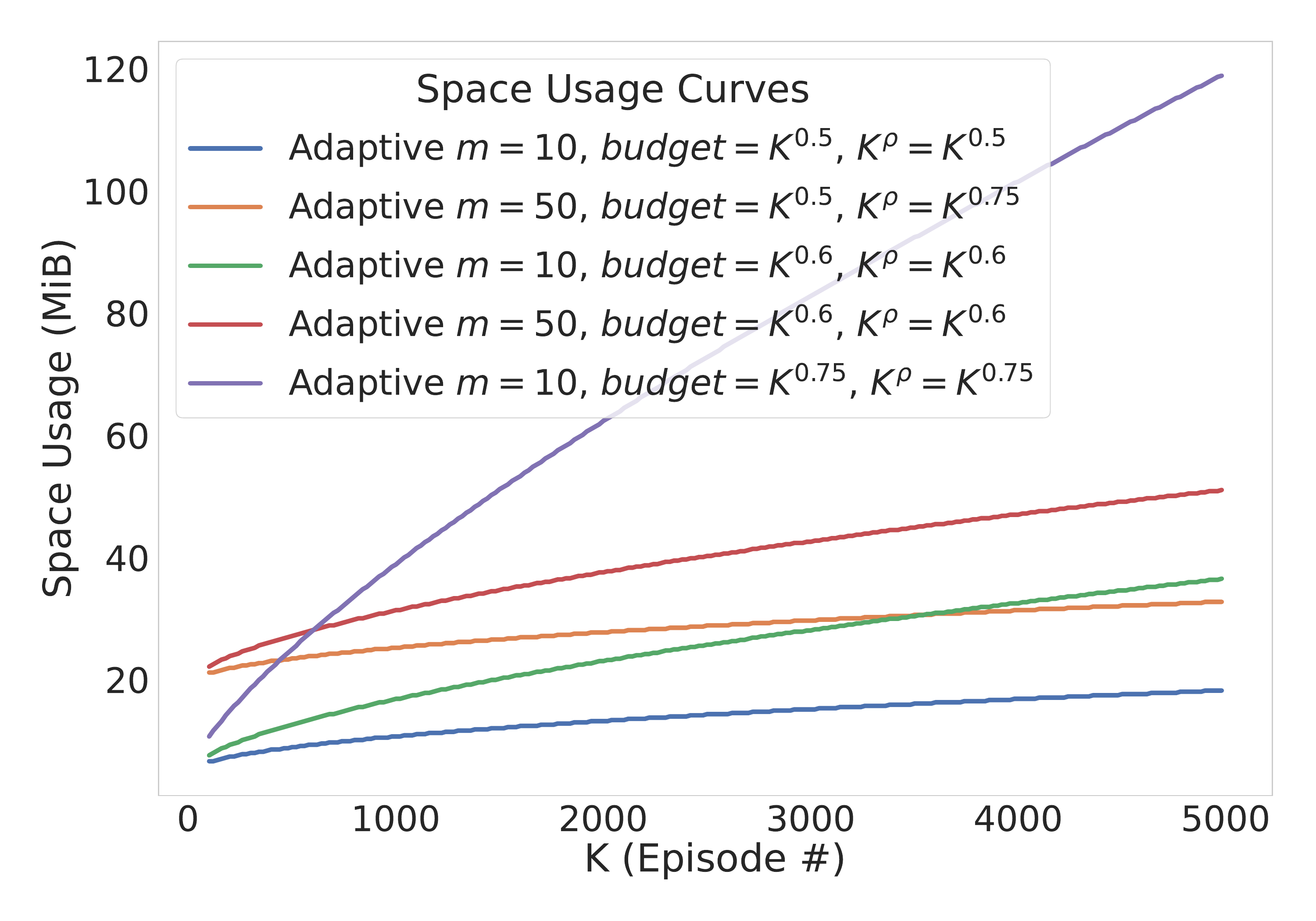}
\caption{Synthetic data: Space usage of \textsc{LSVI-UCB-Adaptive} as a function of parameters.}
\label{fig:space_adapt_param}
\end{figure}

We can see that, unsurprisingly, the space usage of \textsc{LSVI-UCB-Fixed} increases with the power of $K$, so the best space usage is achieved when the learning intervals are of length $\sqrt{K}$. We notice the same behavior for running time, presented in the supplementary material.

\subsection{Linearized Environment Experiments}

We experimented with the basic LSVI-UCB and the \textsc{LSVI-UCB-Fixed} algorithms on the \textsf{Linearized-Alien} and \textsf{Linearized-Phoenix} datasets (see \Cref{sec:expts:envs}).
The experiments are done for $K = 500$ episodes, with the following parameters:
For \textsc{LSVI-UCB-Fixed}, the learning interval size is set to $K^{0.75}$ (i.e., $\rho = 0.75$), and for \textsc{LSVI-UCB-Adaptive}, $m$ is set to $10$, $\tau_c = 0.1$, $\textsf{Budget} = K^{0.5}$, and $\rho = 0.75$.

We have stopped LSVI-UCB algorithm at $K = 280$ as both the time and space requirements were too high compared to our algorithms, which we ran for $K = 500$ episodes.

\Cref{tab:linearized}, \Cref{tab:linearized-2}, and \Cref{tab:linearized-3} show the results for the linearized environments in terms of reward, space usage, and process time at the start of the execution ($K = 100$), when LSVI-UCB has stopped ($K = 280$), and at the end of the execution ($K = 500$).

\begin{table*}
\caption{Results on linearized environments: Total reward.}
\label{tab:linearized}
\centering
\begin{tabular}{ccrrrrrrrrr}
\toprule
\textbf{Dataset} & \textbf{Algorithm} & \multicolumn{3}{c}{\textbf{Total Reward}} \\
&  & $100$ & $280$ & $500$ \\
\midrule
\multirow{3}*{\textsf{Alien}} & \textsc{LSVI-UCB} & {\color{black}$9927.16$} & {\color{black}$27798.32$} \\
& \textsc{LSVI-UCB-Fixed} &
{\color{black}$9927.16$} & {\color{black}$27798.32$} & {\color{black}$49635.66$} \\
& \textsc{LSVI-UCB-Adaptive} & {\color{black}$9927.16$} & {\color{black}$27795.68$} & {\color{black}$49635.66$} \\
\midrule
\multirow{3}*{\textsf{Phoenix}} & \textsc{LSVI-UCB} & {\color{black}$10061.76$} & {\color{black}$28171.04$} & {\color{black}$50306.24$} \\
& \textsc{LSVI-UCB-Fixed} & {\color{black}$10061.76$} & {\color{black}$28171.04$} & {\color{black}$50306.23$} \\
& \textsc{LSVI-UCB-Adaptive} & {\color{black}$10061.76$} & {\color{black}$28171.85$} & {\color{black}$50306.24$} \\
\bottomrule
\end{tabular}
\end{table*}

\begin{table*}
\caption{Results on linearized environments: Space usage.}
\label{tab:linearized-2}
\centering
\begin{tabular}{ccrrrrrrrrr}
\toprule
\textbf{Dataset} & \textbf{Algorithm} & \multicolumn{3}{c}{\textbf{Space Usage (GiB)}} \\
& & $100$ & $280$ & $500$ \\
\midrule
\multirow{3}*{\textsf{Alien}} & \textsc{LSVI-UCB} & {\color{black}$0.988$} & {\color{black}$2.390$} & {\color{black}$-$}\\
& \textsc{LSVI-UCB-Fixed} & {\color{black}$0.716$} & {\color{black}$1.284$} & {\color{black}$1.868$}\\
& \textsc{LSVI-UCB-Adaptive} & {\color{black}$2.610$} & {\color{black}$2.657$} & {\color{black}$2.703$}\\
\midrule
\multirow{3}*{\textsf{Phoenix}} & \textsc{LSVI-UCB} & {\color{black}$0.579$} & {\color{black}$1.243$} & {\color{black}$2.054$}\\
& \textsc{LSVI-UCB-Fixed} & {\color{black}$0.449$} & {\color{black}$0.719$} & {\color{black}$0.995$}\\
& \textsc{LSVI-UCB-Adaptive} & {\color{black}$2.561$} & {\color{black}$2.583$} & {\color{black}$2.605$}\\
\bottomrule
\end{tabular}
\end{table*}

\begin{table*}
\caption{Results on linearized environments: Process time.}
\label{tab:linearized-3}
\centering
\begin{tabular}{ccrrrrrrrrr}
\toprule
\textbf{Dataset} & \textbf{Algorithm} & \multicolumn{3}{c}{\textbf{Process Time (hrs)}}\\
&  & $100$ & $280$ & $500$ \\
\midrule
\multirow{3}*{\textsf{Alien}} & \textsc{LSVI-UCB} & {\color{black}$0.19$} & {\color{black}$1.66$} & {\color{black}$-$}\\
& \textsc{LSVI-UCB-Fixed} & {\color{black}$0.23$} & {\color{black}$1.35$} & {\color{black}$3.19$}\\
& \textsc{LSVI-UCB-Adaptive} & {\color{black}$0.14$} & {\color{black}$0.34$} & {\color{black}$0.72$}\\
\midrule
\multirow{3}*{\textsf{Phoenix}} & \textsc{LSVI-UCB} & {\color{black}$0.184$} & {\color{black}$0.703$} & {\color{black}$2.879$}\\
& \textsc{LSVI-UCB-Fixed} & {\color{black}$0.115$} & {\color{black}$0.580$} & {\color{black}$1.476$}\\
& \textsc{LSVI-UCB-Adaptive} & {\color{black}$0.108$} & {\color{black}$0.207$} & {\color{black}$0.423$}\\
\bottomrule
\end{tabular}
\end{table*}

Note that the adaptive algorithm has a high space overhead (compared to the fixed algorithm) of around $2$ GiB in these experiments, since $d = 512$ for the linearized representations, horizon $H = 50$, and lookback period $m = 10$.
The overhead (which is $\Theta(d^2 M H)$) is for storing the past history.
For large $K$, this overhead would be dominated by the dependency of space usage on $K$.

The results are in line with the results on the synthetic experiments and show that our algorithms are more efficient in terms of space usage and process time compared to LSVI-UCB even in real-world linearized environments.
The space usage for \textsc{LSVI-UCB-Adaptive} is large, as described in the earlier paragraph, for small values of $K$, but the process time is small compared to LSVI-UCB. Moreover, it can be seen that the space usage remains constant, while that of LSVI-UCB grows with $K$.

\section{Conclusion}

\label{sec:conclusion}

In this paper, we  studied \emph{provably} space efficient variants of the LSVI-UCB algorithm, and show that they perform well in terms of accumulated regret.
Our experimental results support our findings, and show that significant gains can be achieved without significantly sacrificing regret for both of the variants.

Going forward, we plan to study the possible space savings obtainable with other reinforcement learning algorithms, and also test our proposed methods on other reinforcement learning benchmarks by using our linearization approach.

\section*{Acknowledgements}

DesCartes:
This research is supported by the National Research Foundation, Prime Minister’s Office, Singapore under its Campus for Research Excellence and Technological Enterprise (CREATE) programme.

We would like to thank the Research Cluster of the School of Computing of the National University of Singapore, for supporting our experiments.

The work of AB was supported in part by National Research Foundation Singapore under its NRF Fellowship Programme (NRF-NRFFAI-2019-0002) and an Amazon Faculty Research Award.

\newcommand{\etalchar}[1]{$^{#1}$}

\appendix

\section{Intuition for \texorpdfstring{\Cref{alg:LSVI-UCB_adaptive_reset}}{Algorithm 3}}

\label{sec:motivation}

In this section we prove \Cref{prop:motivation}.

\subsection{Preliminaries (cont.)}

We recount here some fundamental notions that are needed in this supplementary material.
Let $A$ be a matrix.
We denote by $\Tr\cprn{A}$ the trace of $A$.
We denote by $\bars{A}_p$ the \emph{$L_p$ norm of $A$} for $p\in\mathbb{N}\cup\brkts{\infty}$, by $\bars{A}_{2\to2}$ the \emph{operator norm of $A$}, and by $\bars{A}_{\rm F}$ the \emph{Frobenius norm of $A$}.
If $A$ is positive definite, then we denote this fact by $A \succ 0$.
If $x$ is a vector in $\R^d$ and $A \in \R^{d \times d}$ is a positive definite matrix, we denote the \emph{$A$-ellipsoid norm of $x$} by $\|x\|_A := (x^\top A x)^{\frac12}$.
We denote the \emph{minimum and maximum eigenvalues of $A$} by $\lambda_{\min}\cprn{A}$ and $\lambda_{\max}\cprn{A}$, respectively.

\subsection{Some Useful Results}

\label{sec:useful}

The following results will come handy in \Cref{sec:motivation}.

\begin{lemma}
\label{lem:ellipsoid-inequality}
If $x \in \R^d$ and $A, A^\prime \succ 0 \in \R^{d \times d}$, then
\[
\abs{\bars{x}_A^2 - \bars{x}_{A^\prime}^2} \leq \bars{x}_2^2 \bars{A - A^\prime}_{2 \to 2}.
\]
\end{lemma}

\begin{proof}
We have that
\begin{align*}
\abs{\bars{x}_A^2 - \bars{x}_{A^\prime}^2}
& = (x^\top A x) - (x^\top A^\prime x)\\
& = x^\top (Ax) - x^\top (A^\prime x)\\
& = x^\top ((A - A^\prime) x)\\
& \leq \bars{x}_2 \bars{(A - A^\prime) x}_2\\
& \leq \bars{x}_2^2 \bars{A - A^\prime}_{2 \to 2}.
\qedhere
\end{align*}
\end{proof}

We require the following property of sub-sampling from a large set of independent samples.

\begin{lemma}
\label{lem:subsampling_discrete}
Let $L:=\prn{v_1,\dots,v_n}$ be a sequence of vectors sampled independently from $\cD$, a distribution on $\R^d$.
Let $p$ be a distribution on $[n]$, and denote $\|p\|_2^2 := \sum_{i \in [n]} p(i)^2$.
Then, with probability at least $1 - \delta$, the distribution that arises from sampling $m \leq \sqrt{{\delta}/{\|p\|_2^2}}$ vectors from $L$ independently according to the distribution $p$ is the same as the distribution that arises from sampling $m$ vectors independently from the distribution $\cD$.
\end{lemma}

\begin{proof}
If we sample $z \in L$ with $\cpr{z = v_i} = p(i)$, then for any Borel set $B$ of $\R^d$, we have
\begin{align*}
\cpr{z \in B}
&= \sum_{i=1}^{n}\cpr{z \in B|z = v_i} \cpr{z = v_i}
= \sum_{i=1}^{n} \cpr{v_i \in B} \cdot p(i)
= \cpr{v_1 \in B},
\end{align*}
since the $v_i$'s are identically distributed.
This shows that $z$ has the same distribution $\cD$.
If we sample $z_1, z_2$ from $L$ independently according to $p$ and get $z_1 = v_i$, $z_2 = v_j$ for $i, j \sim p$, the probability that $i = j$ is given by $\|p\|_2^2$.
Hence, if we take $m$ independent samples, the probability that there is at least one collision is upper bounded by $\binom{m}{2} \|p\|_2^2 \leq m^2 \|p\|_2^2$.
So, for $m \leq \sqrt{{\delta}/{\|p\|_2^2}}$, this probability is upper-bounded by $\delta$.
If no collisions happen in the $m$ samples from $p$, then the samples are independent by construction of $L$.
\end{proof}

We now prove a version of \Cref{lem:subsampling_discrete} for continuous distributions.

\begin{lemma}
\label{lem:subsampling_continuous}
Let $L:=\prn{v_i}_{i \in \cI}$ be a sequence of vectors sampled independently from $\cD$, a distribution on $\R^d$, where the index set $\cI$ is \emph{non-discrete}.
Let $p$ be a \emph{continuous} distribution on $\cI$.
Then, with probability $1$, the distribution that arises from sampling $m$ vectors from $L$ independently according to the distribution $p$ is the same as the distribution that arises from sampling $m$ vectors independently from $\cD$.
\end{lemma}

\begin{proof}
If we sample $z \in L$ by sampling $i \sim p$ and taking $z = v_i$, then for any Borel set $B$ of $\R^d$, we have
\[
\cpr{z \in B}
= \int_{x \in \cI} \cpr{z \in B|z = v_x} \cdot p(x)\ \mathrm{d}x
= \cpr{v_1 \in B},
\]
where the last step is valid since since $\cpr{z \in B|z = v_x}$ does not depend on $x$.
This shows that $z$ follows the distribution $\cD$.
If we sample $z_1, \ldots, z_m$ from $L$ independently according to $p$, by the fact that $p$ is continuous, we have that the probability of collisions between $z_1,\ldots,z_m$ is $0$.
In that case, the samples are independent by the construction of $L$.
\end{proof}

We require the following result by~\cite{vershynin2010close}.

\begin{theorem}[See~\cite{vershynin2010close,adamczak2012sharp}]
\label{thm:vershynin}
Consider a random vector $X$ in $\R^d$.
Let $\Sigma$ be the covariance matrix of $X$ and $\Sigma_N := \frac{1}{N} \sum_{i = 1}^N w_i w_i^\top$ be the empirical covariance matrix corresponding to $N$ i.i.d. samples $w_1,\dots,w_N\sim\cN\cprn{0,\Sigma}$.
Then, with probability at least $1 - 2 \exp\cprn{-O\prn{\sqrt{d}}}$, it is the case that
\[
1 - O\cprn{\sqrt{\beta}}
\leq \lambda_{\min}\cprn{\Sigma_N}
\leq 1 + O\cprn{\sqrt{\beta}}
\]
for $\beta:=d/N$.
\end{theorem}

Drawing upon \Cref{thm:vershynin} we show the following lemma.

\begin{lemma}
\label{lem:gaussian_vec-matrix_min_eigenvalue_convergence}
Let $w_1,\dots,w_k$ be drawn from $\cN\cprn{0,\Sigma}$.
Then the minimum eigenvalue of $\sum_{i=1}^k w_iw_i^\top$ is at least $k\cdot\lambda_{\min}\cprn{\Sigma}/100$, with probability at least $1 - 2 \exp\cprn{-O\prn{\sqrt{d}}}$.
\end{lemma}

\begin{proof}
Let $\overline{\Lambda}_k := \frac{1}{k} \sum_{i=1}^{k} w_i w_i^\top$.
It would suffice to show that $\lambda_{\rm min}(\overline{\Lambda}_k) \geq \lambda_{\rm min}(\Sigma)/100$ with high probability.
Consider the Cholesky decomposition of $\Sigma$, that is, $\Sigma = U U^\top$ for some matrix $U$.
We may now write $w_1,\ldots,w_k$ as $U z_1, \ldots, U z_k$, respectively, where $z_1,\ldots,z_k \sim \cN(0,I)$ are i.i.d. samples.

Then we have that
\begin{align*}
\overline{\Lambda}_k
&= \frac{1}{k}\sum_{i=1}^{k} w_i w_i^\top\\
&= \frac{1}{k}\sum_{i=1}^{k} Uz_i z_i^\top U^\top\\
&= U \cprn{\frac{1}{k} \sum_{i=1}^{k} z_i z_i^\top} U^\top\\
&= U D_k U^\top,
\end{align*}
whereby $D_k = \frac{1}{k} \sum_{i=1}^{k} z_i z_i^\top$.
By \Cref{thm:vershynin}, we get that
\[
1 - O\cprn{\sqrt{\beta}}
\leq \lambda_{\min}\cprn{D_k}
\leq 1 + O\cprn{\sqrt{\beta}},
\]
for $\beta:=d/k$ with probability at least $1 - 2 \exp\cprn{-O\prn{\sqrt{d}}}$.

We shall now continue as follows.
For a matrix $A$, the Rayleigh quotient definition of minimum eigenvalue states that $\lambda_{\min}\cprn{A}=\min_{x:\bars{x} = 1} x^\top A x$.
Note that
\[
\bars{U^\top x}^2
=\prn{U^\top x}^\top U^\top x
=x^\top U U^\top x
=x^\top \Sigma x
\geq \lambda_{\min}\cprn{\Sigma}.
\]
Therefore
\begin{align*}
\lambda_{\min}\cprn{\overline{\Lambda}_k}
&=\lambda_{\min}(U D_k U^\top)\\
&=\min_{x:\bars{x} = 1} x^\top U D_k U^\top x\\
&=\min_{x:\bars{x} = 1} \prn{U^\top x}^\top D_k U^\top x\\
&=\lambda_{\min}(\Sigma) \min_{x:\bars{x} = 1} \frac{\prn{U^\top x}^\top}{\sqrt{\lambda_{\min}(\Sigma)}} D_k \frac{U^\top x}{\sqrt{\lambda_{\min}(\Sigma)}}\\
&\geq \lambda_{\min}(\Sigma) \min_{z:\bars{z} = 1} z^\top D_k z\\
&=\lambda_{\min}(D_k) \lambda_{\min}(\Sigma).
\end{align*}
Since
\[
\lambda_{\min}(\Sigma)-O\cprn{\lambda_{\min}(\Sigma)} O\cprn{\sqrt{\beta}}\\
\leq \lambda_{\min}\cprn{D_k} \lambda_{\min}(\Sigma)
\]
by \Cref{thm:vershynin}, we get that
\[
\lambda_{\min}\cprn{\overline{\Lambda}_k}
\geq \lambda_{\rm min}(\Sigma) - O\cprn{\lambda_{\rm min}(\Sigma)} O\cprn{\sqrt{\beta}}.
\]
What is left is to show that $\lambda_{\min}(\Sigma) - O\cprn{\lambda_{\min}(\Sigma)} O\cprn{\sqrt{\beta}}$ is at least $\lambda_{\rm min}(\Sigma)/100$ with high probability.
The latter follows by the value of $\beta = d/k$ and the preceding discussion.
\end{proof}

\subsection{Proof of \texorpdfstring{\Cref{prop:motivation}}{}}

In the following, we shall assume that $H, \Tr\cprn{\Sigma}, \bars{\Sigma^{-1}}_{2\to2}=O_d\cprn{1}$.
Note that these are natural assumptions.

\subsubsection{Convergence of \texorpdfstring{$\Lambda_{h, k}^{-1}$}{} Implies Convergence in the Action-Value Function}

Define
\[
\Lambda_{h,k}
:= \sum_{i=1}^{k-1} \phi(s_{h,i},a_{h,i}) \phi(s_{h,i},a_{h,i})^\top + \lambda I
\]
for any $h$ and $k$.
We will say that \emph{LSVI-UCB (\Cref{alg:LSVI-UCB}) converges at $k_c$} if, for all $k \geq k_c$,
\[
\bars{\Lambda_{h,k}^{-1} - \Lambda_{h,k + 1}^{-1}}_{2 \to 2}
\leq c / k^2,
\]
for some $c > 0$ (that depends on $d$).

Let $w_{h,k}$ be the weight vector learned by LSVI-UCB at step $h$ of episode $k$.
We bound the difference between $w_{w,k}$ and $w_{h,k+1}$.

\begin{lemma}
\label{lem:vec_convergence}
We have that
\begin{align*}
\bars{w_{h,k}-w_{h,k+1}}_2
\leq O_d(1/k),
\end{align*}
for all $k \geq \sqrt{K}$, with high probability.
\end{lemma}

\begin{proof}
Let
\begin{align*}
\psi_{h,k}
&:= \sum_{i = 1}^{k-1}\phi\cprn{s_{h,i},a_{h,i}}r_h\cprn{s_{h,i},a_{h,i}}
+\sum_{i = 1}^{k-1}\phi\cprn{s_{h,i},a_{h,i}}\max_{a \in \cA} Q_{h+1}\cprn{s_{h+1,i},a}
\end{align*}
and
\[
z_{h,k}
:= r_h\cprn{s_{h,k},a_{h,k}} + \max_{a\in\cA} Q_{h+1}\cprn{s_{h+1,k},a}.
\]
Note that $z_{h,k} \in [0,H+1]$ by the bounded reward assumption.
We have that $w_{h,k} = \Lambda_{h,k}^{-1} \psi_{h,k}$, and
\begin{align*}
w_{h,k+1} - w_{h,k}
&= \Lambda_{h,k+1}^{-1} \sum_{i = 1}^{k}\phi\cprn{s_{h,i},a_{h,i}} \big(r_h\prn{s_{h,i},a_{h,i}}
+ \max_a Q_{h+1}\cprn{s_{h+1,i},a}\big) - w_{h,k}\\
&= \Lambda_{h,k+1}^{-1} \prn{\psi_{h,k}
+ z_{h,k}\phi\cprn{s_{h,k},a_{h,k}}} - w_{h,k}\\
&=\cprn{\Lambda_{h,k+1}^{-1} - \Lambda_{h,k}^{-1}} \psi_{h,k}
+z_{h,k}\Lambda_{h,k+1}^{-1} \phi(s_{h,k},a_{h,k}).
\end{align*}
Hence, using the triangle inequality and $0 \leq z_{h,k} \leq H+1$, we have
\begin{align*}
\bars{w_{h,k+1} - w_{h,k}}
&\leq \bars{\cprn{\Lambda_{h,k+1}^{-1} - \Lambda_{h,k}^{-1}} \psi_{h,k}}
+(H+1)\bars{\Lambda_{h,k+1}^{-1} \phi(s_{h,k},a_{h,k})}.
\end{align*}
To bound the first term, note that
\[
\bars{\psi_{h,k}}_2 \leq (H+1)\cdot(k-1)\cdot O(\sqrt{\Tr(\Sigma)}),
\]
and hence by our assumption we have that
\[
\bars{\cprn{\Lambda_{h,k+1}^{-1} - \Lambda_{h,k}^{-1}} \psi_{h,k}}_2
\leq O_d\cprn{\frac{(H+1)\sqrt{\Tr(\Sigma)}}{k}}.
\]
To bound the second term, note that $\Lambda_{h,k+1}^{-1}$ is equal to
\begin{align*}
\prn{\sum_{i = 1}^{k} \phi\cprn{s_{h,i},a_{h,i}}\phi\cprn{s_{h,i},a_{h,i}}^\top+\lambda I}^{-1}
&\approx_{\rm F} \prn{\sum_{i = 1}^{k}\phi\cprn{s_{h,i},a_{h,i}}\phi\cprn{s_{h,i},a_{h,i}}^\top}^{-1} \\
&\approx_{\rm F}\prn{k\Sigma}^{-1}
=\frac{1}{k}\Sigma^{-1},
\end{align*}
by setting $\lambda=o\cprn{1}$.

So we get that $\Lambda_{h,k+1}^{-1} \phi(s_{h,k},a_{h,k})
\approx_{\rm F} \frac{1}{k}\Sigma^{-1} \phi(s_{h,k},a_{h,k})$.
Let us now take into account the fact that (with high probability) $\bars{\phi(s_{h,k},a_{h,k})}_2\leq O\cprn{\sqrt{\Tr\cprn{\Sigma}}}$ and the assumption that $\bars{\Sigma^{-1}}_{2\to 2} = O_d\cprn{1}$.
In this case, we have
\[
\bars{\Lambda_{h, k + 1}^{-1} \phi(s_{h,k}, a_{h,k})}_2
\leq \frac{1}{k} \cdot O_d\cprn{\sqrt{\Tr\cprn{\Sigma}}},
\]
and so by the discussion above we get that

\begin{align*}
\bars{w_{h,k}-w_{h,k+1}}_2
&\leq (H+1)\cdot O_d\cprn{\frac{\sqrt{\Tr(\Sigma)}}{k}}
+(H+1)\cdot\frac{1}{k}\cdot O_d\cprn{\Tr\cprn{\Sigma}}.
\end{align*}
This concludes the proof as $H,\Tr\cprn{\Sigma}=O_d\cprn{1}$.
\end{proof}

The above yield a bound on $\bars{Q_{h,k}-Q_{h,k'}}_\infty$, as follows.

\begin{lemma}
\label{lem:weight_convergence}
If LSVI-UCB converges at $\sqrt{K}$, then for all $h \in [H]$ and $k'\geq k \geq \sqrt{K}$, we have that
\[
\bars{Q_{h,k}-Q_{h,k'}}_\infty
\leq O_d\cprn{\log k' - \log k + \frac{\beta}{K^{1/4}}},
\]
with high probability.
\end{lemma}

\begin{proof}
For any $k \in [K]$, $h \in [H]$ and $(s, a) \in \cS \times \cA$, we have
\[
Q_{h,k}\cprn{s,a}
:= \min\cprn{w_{h,k}^\top\phi\cprn{s,a}+\beta\bars{\phi\cprn{s,a}}_{\Lambda_{h,k}^{-1}},H}.
\]
If both $Q_{h,k}(s,a)$ and $Q_{h,k+1}(s,a)$ are equal to $H$, the bound in the lemma is trivially satisfied for the $(s, a)$ pair.
Without loss of generality, we have that $Q_{h,k}(s,a)$ is equal to $w_{h,k}^\top\phi\cprn{s,a}+\beta\prn{\phi\cprn{s,a}\Lambda_{h,k}^{-1}\phi\cprn{s,a}}^{1/2}$.
Then we have
\begin{align}
& \abs{Q_{h,k}(s,a) - Q_{h,k+1}(s,a)}\nonumber\\
& \qquad \leq \abs{w_{h,k}^\top\phi\cprn{s,a}+\beta\bars{\phi\cprn{s,a}}_{\Lambda_{h,k}^{-1}}
-\,w_{h,k+1}^\top\phi\cprn{s,a}-\beta\bars{\phi\cprn{s,a}}_{\Lambda_{h,k+1}^{-1}}}\nonumber\\
& \qquad \leq \abs{(w_{h,k} - w_{h,k+1})^\top\phi\cprn{s,a}}
+\beta\abs{\bars{\phi\cprn{s,a}}_{\Lambda_{h,k}^{-1}} - \bars{\phi\cprn{s,a}}_{\Lambda_{h,k+1}^{-1}}}.
\label{eq:second-term}
\end{align}
By Cauchy-Schwarz, the first term of the RHS of \Cref{eq:second-term} is at most
\[
\bars{w_{h,k} - w_{h,k+1}}_2 \bars{\phi(s,a)}_2.
\]
By \Cref{lem:ellipsoid-inequality}, the second term of the RHS of \Cref{eq:second-term} is at most
\[
\beta \cdot \bars{\phi(s,a)}_2^2 \cdot \bars{\Lambda_{h,k}^{-1} - \Lambda_{h,k+1}^{-1}}_{2 \to 2}.
\]
Therefore, by the assumptions of the lemma, we get
\begin{align*}
\abs{Q_{h,k}(s,a) - Q_{h,k+1}(s,a)}
&\leq c_1 \bars{\phi(s,a)}_2 / k
+ c_2 \beta \bars{\phi(s,a)}_2^2 / k^2,
\end{align*}
whereby the RHS on this inequality is $O_d\cprn{1/k + \beta / k^2}$, by taking into account the fact that $\bars{\phi(s,a)}_2=O\cprn{\Tr\cprn{\Sigma}}$ with high probability, as $\phi(s,a)$ is drawn from the Gaussian distribution $\cN\cprn{0,\Sigma}$, and the assumption that $\Tr\cprn{\Sigma}=O_d\cprn{1}$.

Therefore, we have
\begin{align*}
\abs{Q_{h,k}(s,a) - Q_{h,k'}(s,a)}
\leq \sum_{i=k}^{k' - 1} \abs{Q_{h,i}(s,a) - Q_{h,i+1}(s,a)}
\leq \sum_{i=k}^{k' - 1} O_d\cprn{1/i + \beta / i^2}
\end{align*}
Note that when $i \geq K^{3/4}$, we get that $\beta/i^2 \leq \beta/K^{1.5}$.
Then, since $k \geq \sqrt{K}$ by assumption, we get
\[
\sum_{i=k}^{K} \frac{\beta}{i^2}
\leq K^{3/4} \frac{\beta}{K} + \frac{\beta}{\sqrt{K}}
= O\cprn{\frac{\beta}{K^{1 / 4}}}.
\]
Thus,
\begin{align*}
\abs{Q_{h,k}(s,a) - Q_{h,k'}(s,a)}
&\leq O_d\left(\log k' - \log k + \frac{\beta}{K^{1/4}}\right). 
\end{align*}
This concludes the proof.
\end{proof}

We define $\widehat{Q}$ to be the action-value function learned by \Cref{alg:LSVI-UCB_adaptive_reset}, whereas $Q$ denotes the action-value function used by LSVI-UCB (\Cref{alg:LSVI-UCB}).
Note that some of the values of $\widehat{Q}$ are identical to $Q$ (in the learning intervals) and some might not be (in the rest of the episodes).

We use \Cref{lem:weight_convergence} to show the closeness $Q$ and $\widehat{Q}$ in the non-learning intervals.
If episode $k \in [K]$ is in a non-learning interval, let $k^\prime < k$ denote the episode at the end of the last learning interval.
By construction of \Cref{alg:LSVI-UCB_adaptive_reset}, we have $\widehat{Q}_{h,k} = \widehat{Q}_{h,k^\prime} = Q_{h,k^\prime}$.
By applying \Cref{lem:weight_convergence}, we get the following corollary.

\begin{corollary}
\label{cor:action-value_difference}
If \Cref{alg:LSVI-UCB_fixed_reset} converges at $\sqrt{K}$, then for all  $k > \sqrt{K}$, we have that
\[
\bars{Q_{h,k}-\widehat{Q}_{h,k}}_\infty
\leq \sum_{\ell=k^\prime}^{k-1}\bars{Q_{h,\ell+1}-Q_{h,\ell}}_\infty
\leq O_d\cprn{\log k - \log k' + \frac{\beta}{K^{1 / 4}}},
\]
with high probability.
\end{corollary}

Finally, we bound the maximum difference between $Q$ and $\widehat{Q}$.

\begin{lemma}
\label{lem:action-value_difference}
For all $\varepsilon$, assuming $\bars{Q-\widehat{Q}}_\infty\leq\varepsilon$ (with high probability), if it is the case that $\prn{s',a'}=\arg\max_{\prn{s,a}}Q\cprn{s,a}$ and $\prn{s'',a''}=\arg\max_{\prn{s,a}}\widehat{Q}\cprn{s,a}$, then it is the case that $\abs{Q\cprn{s',a'}-\widehat{Q}\cprn{s'',a''}}\leq\varepsilon$ (with high probability).
\end{lemma}

\begin{proof}
We have $\widehat{Q}\cprn{s',a'} \in \sqbra{{Q}(s',a') - \varepsilon,{Q}(s',a') + \varepsilon}$ by assumption, therefore $\widehat{Q}(s'',a'') \geq Q(s',a') - \varepsilon$ since $(s'',a'')$ gives the maximum value of $\widehat{Q}$ by definition.

Similarly, $Q\cprn{s'',a''} \in \sqbra{\widehat{Q}(s'',a'') - \varepsilon,\widehat{Q}(s'',a'') + \varepsilon}$, so that $Q(s',a') \geq \widehat{Q}(s'',a'') - \varepsilon$ if and only if $\widehat{Q}(s'',a'') \leq \widehat{Q}(s',a') + \varepsilon$.
This concludes the proof.
\end{proof}

That is, we have shown that $\bars{\Lambda_{h,k}^{-1} - \Lambda_{h,k + 1}^{-1}}_{2 \to 2} \leq c / k^2$ (for sufficiently large $k$) implies that the action-value function of \Cref{alg:LSVI-UCB_adaptive_reset}, namely $\widehat{Q}$, is close to the action-value function of \Cref{alg:LSVI-UCB}, namely $Q$.

\subsubsection{Convergence of \texorpdfstring{$\Lambda_{h, k}^{-1}$}{} Under Some Assumptions}

We begin by stating some assumptions.

\begin{assumption}
\label{assump:fixed_transition}
At each step $h$ in each episode $k$, the state $s_{h,k}$ is sampled from a distribution $p_\cS$ on $\cS$.
That is, for all $(s, a)$ and $h$ it is the case that $\P_h(s, a) = p_\cS$.
Also, either $(a)$ $\cS$ is non-discrete (e.g., $\R^{s}$) and $p_\cS$ is a continuous distribution on $\cS$, or $(b)$ $\cS$ is finite and $\bars{p_\cS}_2^2 \leq 1/T^4$.
\end{assumption}

Note that \Cref{assump:fixed_transition} is not (necessarily) true for general linear MDPs, but it is motivated by the notion of \emph{occupancy measure} $\nu^\pi = \sum_{t = 0}^\infty \gamma^t \P^\pi\cprn{s_t, a_t}$ (for some discount factor $0 < \gamma < 1$), relative to a policy $\pi$, and the potential practical use of generative models (such as deep learning networks) to sample states given some state-action pair as input.

\begin{assumption}
\label{assump:gaussian_features}
For each $(s, a) \in \cS \times \cA$, the feature-vector $\phi(s, a)\in\R^d$ is drawn independently from a non-degenerate Gaussian distribution $\cN(0, \Sigma)$, where $\Sigma \succ 0$ is positive-definite matrix such that $\lambda_{\min}(\Sigma) = c_{\min} > 0$.
\end{assumption}

We now show the following straightforward proposition.

\begin{proposition}
\label{lem:lambda_convergence}
With probability at least $1 - 1/K$, for all $h \in [H]$ and $k \in [K-1]$, we have
\[
\bars{\Lambda_{h,k}^{-1} - \Lambda_{h,k + 1}^{-1}}_{2 \to 2}
\leq O_d(1/k^2).
\]
\end{proposition}

\begin{proof}
By \Cref{assump:fixed_transition} and \Cref{assump:gaussian_features}, using either \Cref{lem:subsampling_discrete} or \Cref{lem:subsampling_continuous} (depending on whether $\cS$ is discrete or continuous), for any $h \in [H]$, with probability at least $1 - 1/T^2$, the values $\phi(s_{h,1}, a_{h,1}),\dots,\phi(s_{h,K},a_{h,K})$ are i.i.d samples from $\cN(\mu,\Sigma)$.
By a union bound, this holds for all $h$ simultaneously with probability $1 - 1/K^2$.

By construction, we have that $\Lambda_{h,k+1} = \Lambda_{h,k} + \phi(s_{h,k}, a_{h,k}) \phi(s_{h,k}, a_{h,k})^\top$, so the Sherman-Morrison-Woodbury identity gives, for $u := \phi(s_{h,k}, a_{h,k})$ and $w := \Lambda_{h,k}^{-1} \phi(s_{h,k}, a_{h,k})$, that
\begin{align*}
\bars{\Lambda_{h,k}^{-1} - \Lambda_{h,k + 1}^{-1}}_{2 \to 2}
&= \bars{\frac{\Lambda_{h,k}^{-1} u u^\top \Lambda_{h,k}^{-1}}{1 + u^\top \Lambda_{h,k}^{-1} u}}_{2 \to 2}\\
&= \bars{\frac{w w^\top}{1 + u^\top \Lambda_{h,k}^{-1} u}}_{2 \to 2}\\
&= \frac{\bars{w w^\top}_{2 \to 2}}{{\bars{1 + u^\top \Lambda_{h,k}^{-1} u}}_{2 \to 2}}\\
&\leq \bars{w w^\top}_{2 \to 2}\\
&=\|w\|_2^2\\
&=\|\Lambda_{h,k}^{-1} \cdot \phi(s_{h,k}, a_{h,k})\|_2^2\\
&\leq \bars{\Lambda_{h,k}^{-1}}_{2 \to 2}^2 \bars{\phi(s_{h,k},a_{h,k})}_2^2,
\end{align*}
by the fact that $\bars{1 + u^\top \Lambda_{h,k}^{-1} u}_{2 \to 2}\geq1$.

What is left is to bound the quantities $\bars{\Lambda_{h,k}^{-1}}_{2 \to 2}^2$ and $\bars{\phi(s_{h,k},a_{h,k})}_2^2$.
To bound $\bars{\Lambda_{h,k}^{-1}}_{2 \to 2}^2$ our approach is to lower bound the smallest eigenvalue of $\Lambda_{h,k}$ by $\Omega(k)$, with high probability.
By the definition of $\Lambda_{h,k}$ (see \Cref{alg:LSVI-UCB_fixed_reset}) and \Cref{lem:gaussian_vec-matrix_min_eigenvalue_convergence}, we get that the minimum eigenvalue of $\Lambda_{h,k}$ is at least $\prn{k-1}\lambda_{\min}\cprn{\Sigma}+\lambda=\Omega\cprn{k}$ (by appropriately setting $\lambda$).

We now turn to bounding $\bars{\phi\cprn{s_{h,k},a_{h,k}}}_2$.
Since $\phi\cprn{s_{h,k},a_{h,k}}\sim\cN\cprn{\mu,\Sigma}$, it is a standard result that $\bars{\phi\cprn{s_{h,k},a_{h,k}}}_2=O\cprn{\sqrt{\Tr\cprn{\Sigma}}}$, with high probability.
The result follows from the fact that $\Tr\cprn{\Sigma}=O_d\cprn{1}$.
\end{proof}

\section{Other Plots}

\label{sec:plots}

In \Cref{fig:regret_fixed_param} and \Cref{fig:regret_adapt_param} we examine the effects of various tuning parameters on the regret, for \textsc{LSVI-UCB-Fixed} and \textsc{LSVI-UCB-Adaptive}, respectively. Note that there is no significant effect on the regret.

\begin{figure}[ht]
\centering
\includegraphics[width=0.65\columnwidth]{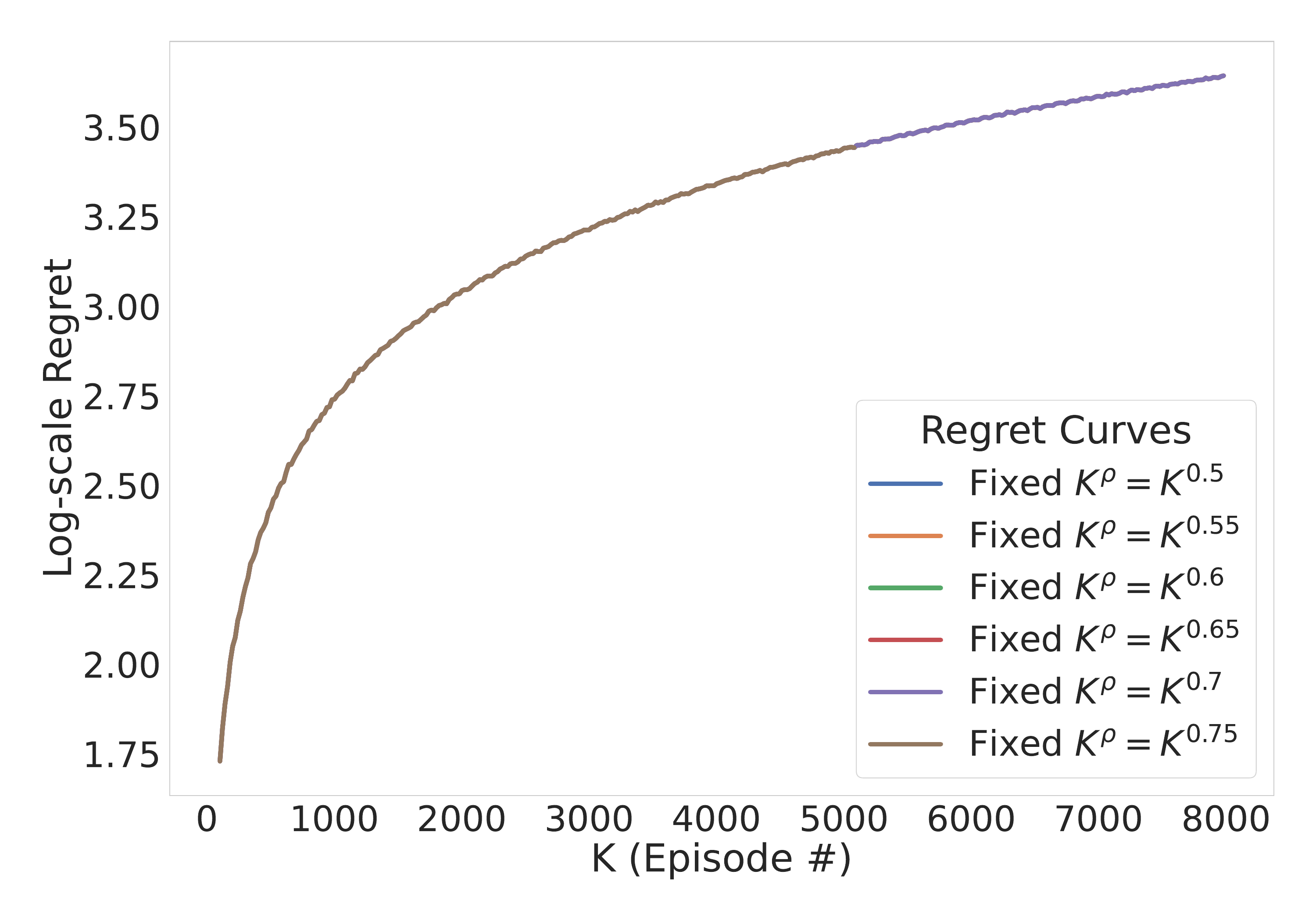}
\caption{Synthetic data: Regret of \textsc{LSVI-UCB-Fixed} as a function of parameters.}
\label{fig:regret_fixed_param}
\end{figure}

\begin{figure}[ht]
\centering
\includegraphics[width=0.65\columnwidth]{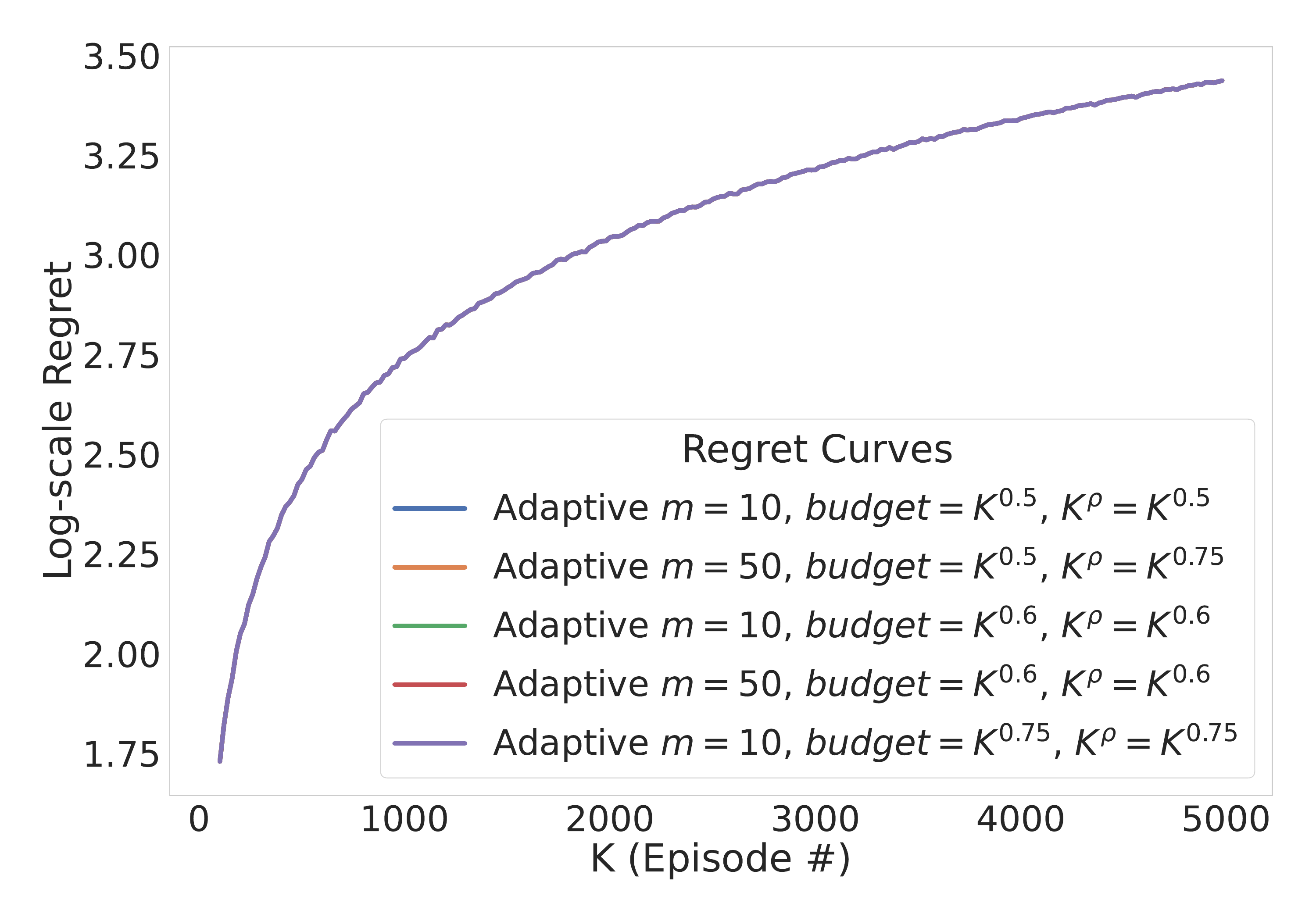}
\caption{Synthetic data: Regret of \textsc{LSVI-UCB-Adaptive} as a function of parameters.}
\label{fig:regret_adapt_param}
\end{figure}

In \Cref{fig:process_time_fixed_param} and \Cref{fig:process_time_adapt_param} we inspect the effects of various tuning parameters on the process time, for \textsc{LSVI-UCB-Fixed} and \textsc{LSVI-UCB-Adaptive}, respectively.

\begin{figure}[ht]
\centering
\includegraphics[width=0.65\columnwidth]{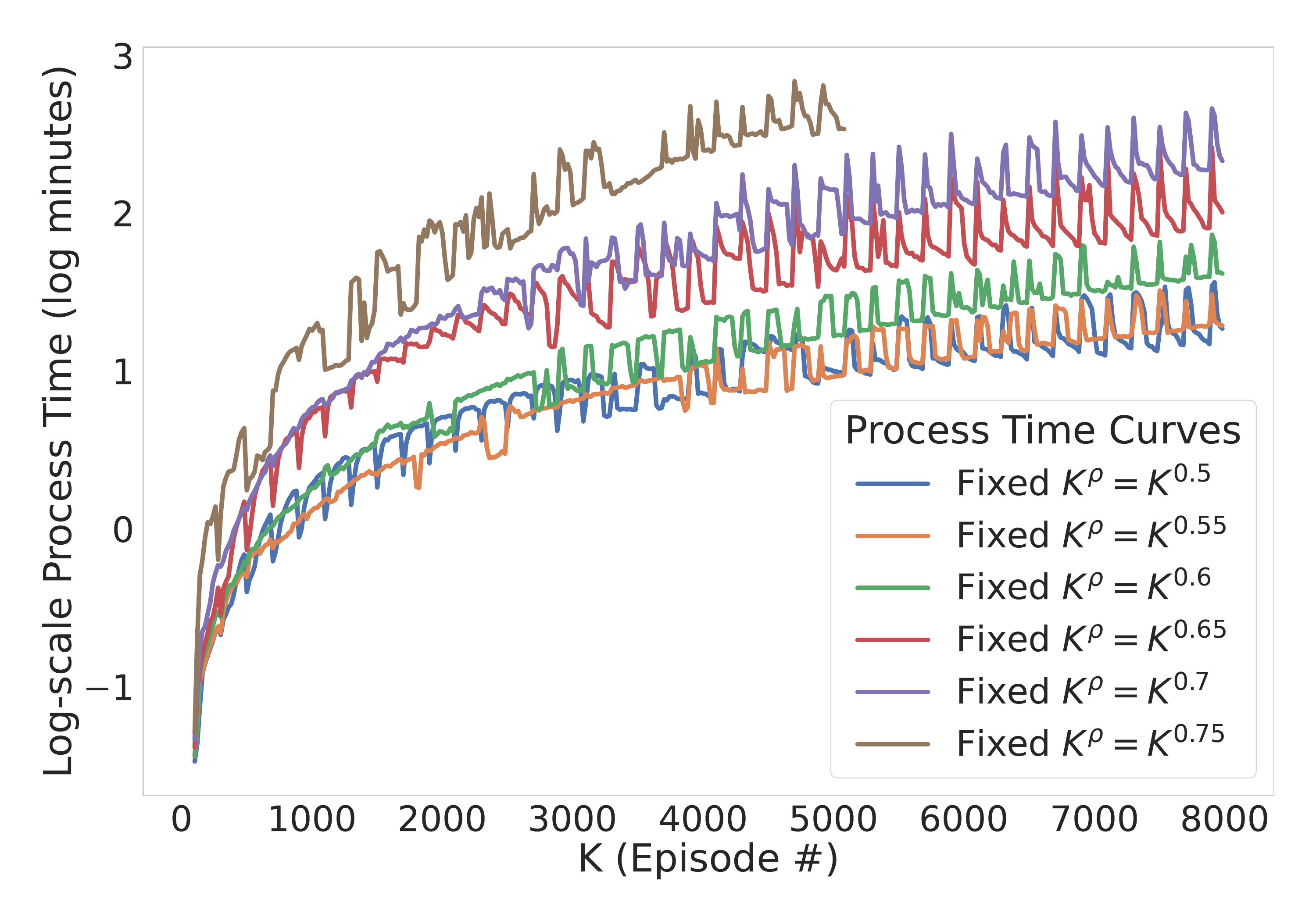}
\caption{Synthetic data: Process time of \textsc{LSVI-UCB-Fixed} as a function of parameters.}
\label{fig:process_time_fixed_param}
\end{figure}

\begin{figure}[ht]
\centering
\includegraphics[width=0.65\columnwidth]{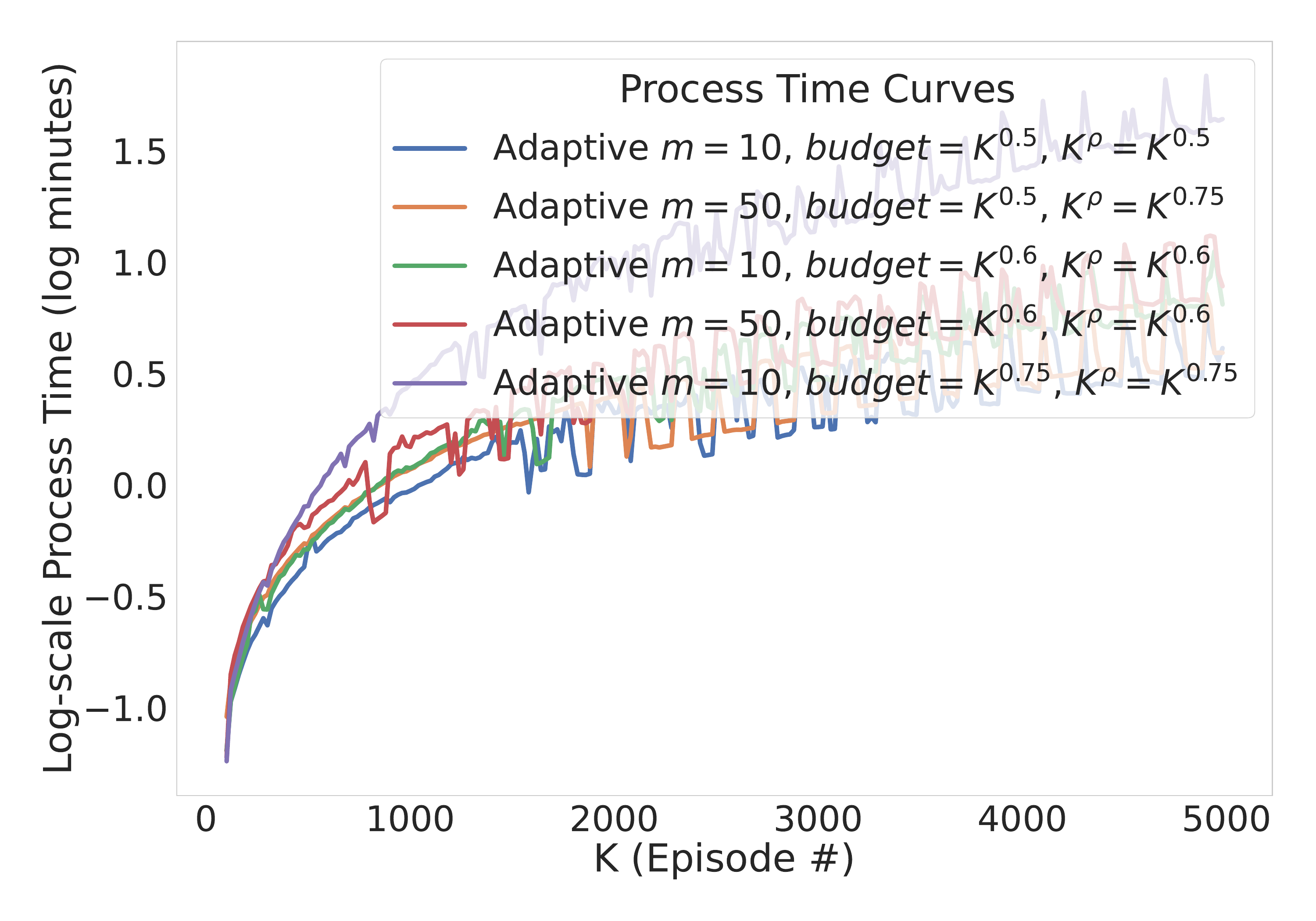}
\caption{Synthetic data: Process time of \textsc{LSVI-UCB-Adaptive} as a function of parameters.}
\label{fig:process_time_adapt_param}
\end{figure}

For \textsc{LSVI-UCB-Fixed}, we can see that, unsurprisingly, the time needed increases with the setting of the power of $K$, so the best running time is achieved when the learning intervals are of length $\sqrt{K}$.
In the case of \textsc{LSVI-UCB-Adaptive}, the process time is more stable with respect to the parameter $\tau$, however a budget of $\sqrt{K}$ seems the best.

\end{document}